%% file: main.tex
\documentclass[10pt,twocolumn,letterpaper]{article}

\usepackage{picture}

\usepackage[pagenumbers]{cvpr} 

\usepackage{comment}
\usepackage{mathtools}
\usepackage{amsmath}
\usepackage{amsthm}
\usepackage{graphicx}
\usepackage{caption}
\usepackage{calc}
\usepackage{stmaryrd}
\usepackage{scalerel}
\usepackage[symbol]{footmisc}

\input{preamble}

\definecolor{cvprblue}{rgb}{0.21,0.49,0.74}
\usepackage[pagebackref,breaklinks,colorlinks,citecolor=cvprblue]{hyperref}

\newcommand{\R}{{\ensuremath{\mathbb R}}}
\newcommand{\C}{{\ensuremath{\mathbb C}}}
\newcommand{\E}{{\ensuremath{\mathbb E}}}

\newcommand{\diff}{{\ensuremath{\gamma}}}

\newcommand{\TND}[1]{{\ensuremath{\mathcal{TN}_{#1}\left(0, I_d \right)}}}

\newcommand{\nice}[2]{{\ensuremath{L^2(#1,  #2)}}}

\newcommand{\vn}[1]{\ensuremath{ { \mathbf{ #1}}}} 

\newcommand{\neigh}[1]{\ensuremath{ { {B}_{#1}}}} 

\DeclarePairedDelimiterX{\abs}[1]{\lvert}{\rvert}{#1} 
\DeclarePairedDelimiterX{\norm}[1]{\lVert}{\rVert}{#1} 
\newcommand{\eval}[2]{\ensuremath{\left. #1 \right\rvert_{#2}}} 

\newcommand{\PT}[2]{\ensuremath{{\mathcal P}_{#1 \shortleftarrow #2}}} 

\newcommand{\enc}{\ensuremath{{\mathcal{ E}}}} 
\newcommand{\dec}{\ensuremath{{\mathcal{ D}}}} 
\newcommand{\nfield}{\ensuremath{{\mathcal{ F}}}} 

\newtheorem{claim}{Claim} 

\DeclarePairedDelimiter{\rightbarinner}{.}{\rvert}
\NewDocumentCommand{\rightbar}{som}{%
  \IfBooleanTF{#1}
    {\rightbarinner*{#3}}
    {\IfNoValueTF{#2}{#3\rvert}{\rightbarinner[#2]{#3}}}%
}

\newif\ifeqSep

\newcommand{\labsym}[3]{
    \ifx#2\empty
    \stackrel{\phantom{#3}}{#1}
    \else
    \stackrel{#2}{\mathmakebox[\widthof{$\stackrel{#3}{#1}$}]{#1}}
    \fi
}

\def\paperID{} 
\def\confName{CVPR}
\def\confYear{2024}

\title{Single Mesh Diffusion Models with Field Latents for Texture Generation}

\author{Thomas W. Mitchel$^{1,2}\footnotemark[2]$ \quad Carlos Esteves$^1$ \quad Ameesh Makadia$^1$ \vspace{0.075in} \\ 
$^1$Google Research \quad $^2$PlayStation \vspace{0.075in} \\
}

\begin{document}

\twocolumn[{
\renewcommand\twocolumn[1][]{#1}
\maketitle
\begin{center}
    \centering
    \captionsetup{type=figure}
    \begin{picture}(\linewidth,0.83\columnwidth)
    \put(0,-5){{\includegraphics[width=\linewidth]{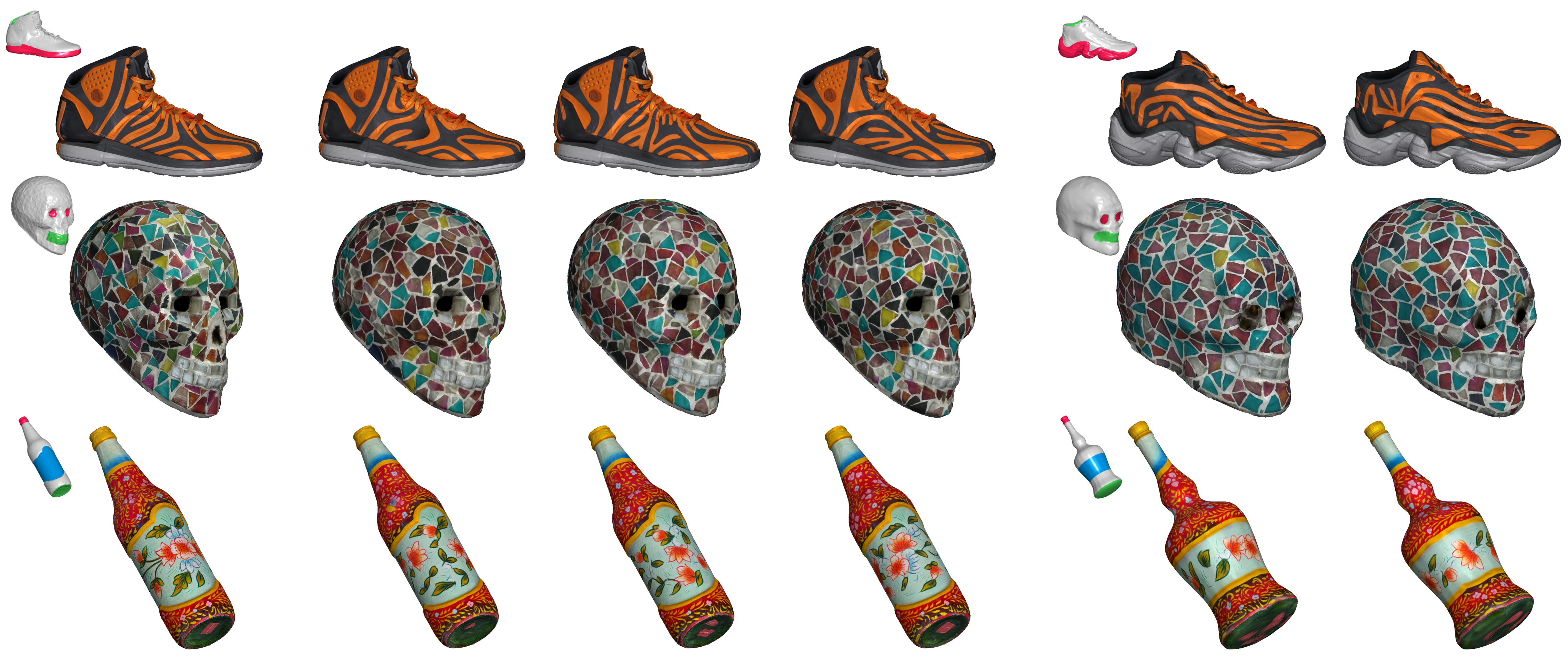}}}
    \thicklines
    \put(31, 209){ {\small {Input}}}
    \put(0, 205){\line(1, 0){85}}
    \put(170, 209){ {\small Generated textures}}
    \put(99, 205){\line(1,0){220}} 
    \put(350, 209){{\small Generative transfer to new geometry}}
    \put(334, 205){\line(1,0){162}} 
    \end{picture}
    \captionof{figure}{Our latent diffusion models operate directly on the surfaces of 3D shapes, synthesizing new high-quality textures (center) after training on a single example (left). Both our novel latent representation and diffusion models are isometry-equivariant, facilitating a notion of generative texture transfer by sampling pre-trained models on new geometries (right). Above our models are conditioned on coarse semantic labels  reflecting a subjective distribution of content, which delineate the sole and interior of the shoes, the eyes and mouths of the skulls, and the decals on the bottles.}
    \label{teaser_fig}
\end{center} 
 }]
\footnotetext[2]{Work done while at Google Research.}

\input{sections/abstract}

\input{sections/intro.tex}

\input{sections/related}

\input{sections/overview}

\input{sections/latents}

\input{sections/diffusion}

\input{sections/experiments}

\input{sections/conclusion}

{
    \small
    \bibliographystyle{ieeenat_fullname}
    \bibliography{main}
}

\clearpage
\setcounter{page}{1}
\maketitlesupplementary

\input{supplement/equivariance}
\input{supplement/architectures}
\input{supplement/more_results}

\end{document}

%% file: preamble.tex
\usepackage[dvipsnames]{xcolor}


%% file: sections/abstract.tex
\begin{abstract}
We introduce a framework for intrinsic latent diffusion models operating directly on the surfaces of 3D shapes, with the goal of synthesizing high-quality textures. Our approach is underpinned by two contributions: Field Latents, a latent representation encoding textures as discrete vector fields on the mesh vertices, and Field Latent Diffusion Models, which learn to denoise a diffusion process in the learned latent space on the surface. We consider a single-textured-mesh paradigm, where our models are trained to generate variations of a given texture on a mesh. We show the synthesized textures are of superior fidelity compared those from existing single-textured-mesh generative models. Our models can also be adapted for user-controlled editing tasks such as inpainting and label-guided generation. The efficacy of our approach is due in part to the equivariance of our proposed framework under isometries, allowing our models to seamlessly reproduce details across locally similar regions and opening the door to a notion of generative texture transfer.  Code and visualizations are available at \url{https://single-mesh-diffusion.github.io/}.
\end{abstract}

%% file: sections/intro.tex
\section{Introduction} 
\label{sec:intro}
The emergence of latent diffusion models (\textbf{LDMs}) \cite{rombach2022high} as powerful tools for 2D content creation has motivated efforts  to replicate their success across different modalities. A particularly attractive direction is the synthesis of textured 3D assets, due to the high cost of building photorealistic objects which can require intricately configured scanning solutions~\cite{downs2022google} or expertise in 3D modeling.

The majority of 3D synthesis methods employing LDMs operate in an image-based setting where the 3D representation is optimized through differentiably rendered images, often in conjunction with score distillation sampling (\textbf{SDS})~\cite{poole2022dreamfusion, metzer2023latent, richardson2023texture, chen2023text2tex, qian2023magic123, cao2023texfusion, tang2023textguided, wang2023breathing}. This design allows the models to leverage the full capabilities of pre-trained 2D LDMs.  A second, alternative method is to rasterize both geometry and texture onto a 3D grid which are compressed into triplane latent features over which the DM operates \cite{wu2023sin3dm}. 
However, in applications where texture synthesis takes precedence over 3D geometry, both approaches sacrifice fidelity in favor of a workable representation --- models relying on planar renderings suffer from view inconsistencies when mapping back to the 3D shape while rasterization inherently aliases fine textural details.

Here, we observe that the discretization of a surface as a triangle mesh is itself a tractable and effective representational space. To this end, we present an original framework for intrinsic LDMs operating directly on the surfaces of shapes, with the goal of generating high-quality textures.

Our method consists of two distinct novel components, serving as the main contributions of our work. The first is a latent representation called \textit{Field Latents} (\textbf{FLs}). Textures are mapped to tangent vector features at the vertices which are responsible for characterizing the local texture. The choice of tangent vector features over scalars allows for the capture of directional information related to the local texture which we show enables superior quality reconstructions. More generally, FLs offer an effective form of perceptual compression in which a high-resolution texture, analogous to a continuous signal defined over a surface, is mapped to a collection of discrete vector fields taking values at the vertices of a lower-resolution mesh. The second component is a \textit{Field Latent Diffusion Model} (\textbf{FLDM}), which learns a denoising diffusion process in field latent space.  Our FLDMs are built upon field convolutions \cite{mitchel2021field}, surface convolution operators designed specifically to process tangent vector features.

Although we outline a general framework for constructing diffusion models for surfaces, inspired by the success of single-image generative models~\cite{shaham2019singan, nikankin2022sinfusion, wang2022sindiffusion, kulikov2023sinddm}, we tailor and deploy our FLDM architecture in the single-textured-mesh setting. We address the problem of generating variations of a given texture on a mesh (Figure~\ref{teaser_fig}, left and center), as well as tasks supporting user control such as label-guided generation and inpainting. The single-textured-mesh setting also lets us circumvent the data scarcity issue. Large-scale 3D datasets with high-quality textures have only recently become available \cite{deitke2022objaverse, deitke2023objaversexl}, with only a fraction of samples across disparate categories possessing complex non-uniform textures. Furthermore, many textured 3D assets that are handcrafted or derived from real-world scans are geometrically or stylistically unique, upon which large-scale models can be challenging to condition.

Experiments reveal our models enable flexible synthesis of high-resolution textures of qualitatively superior fidelity compared to existing single-textured-asset diffusion models, while simultaneously sidestepping the challenges of using 2D-to-3D approaches. Furthermore, both FLs and FLDMs are isometry-equivariant --- each commute with distance-preserving shape deformations. As a result, we find that our approach can be used for \textit{generative texture transfer}, wherein a FLDM trained on a single textured mesh can be sampled on a second similar mesh to texture it in the style of the first (Figure~\ref{teaser_fig}, right).

%% file: sections/related.tex
\section{Related Work}
\label{sec:related}

Existing generative models that synthesize textured 3D assets typically map content to an internal 2D representation which is computationally preferable to working directly over 3D Euclidean space. Perhaps the most popular approach involves iteratively rendering objects from different viewpoints, enabling the use of pre-trained 2D LDMs (\textit{e.g.} Stable Diffusion \cite{rombach2022high}) either indirectly as optimization priors \cite{poole2022dreamfusion, liu2023meshdiffusion, qian2023magic123, metzer2023latent} or to directly apply a full denoising process \cite{richardson2023texture, chen2023text2tex, cao2023texfusion, wang2023breathing}. The former case, pioneered by DreamFusion \cite{poole2022dreamfusion}, uses score distillation sampling (SDS), in which a neural radiance field is optimized such that the rendered images appear to be reasonable samples from a pre-trained LDM; in the latter, LDMs with depth conditioning are used to iteratively denoise the rendered images which are aggregatively projected to the texture map. However, SDS-based methods are computationally feasible only with low-resolution LDMs and small NeRFs, and are unable to synthesize fine details. Additionally, approaches that directly denoise and project must contend with artifacts arising from view inconsistencies and synthesized textures may contain unnatural tones or lighting effects as residuals from their image-trained LDM backbone. 

Alternative representational spaces to drive textured 3D synthesis have also been explored.
Texture Fields \cite{oechsle2019texture, gao2022get3d} encode textures as 3D neural fields and triplane feature maps have emerged recently as an efficient domain for 3D geometry \cite{wu2023sin3dm, zhang2023getavatar, wei2023taps3d}. In particular, Sin3DM \cite{wu2023sin3dm} encodes geometry and texture into an implicit triplane latent representation, then trains and samples a 2D DM on the resulting feature maps. However, geometry and texture must first be rasterized to a 3D grid before triplane convolutions can be employed, which can lead to aliasing of high-fidelity textures and geometric details due to memory constraints on the maximal grid resolution. To disentangle geometry from texture, several methods map textures to domains such as the image plane \cite{chen2022auv, svitov2023dinar, pavllo2021learning, yu2021learning} or sphere \cite{cheng2023tuvf}. Point-UV diffusion \cite{yu2023texture} trains and samples 2D DMs directly in the UV-atlas, though this approach suffers on meshes with disconnected or fragmented UV mappings.  

Similar to our approach, a third class of models encode textures as features at the simplices of a mesh. GAN-based methods Texturify \cite{siddiqui2022texturify} and Mesh2Tex \cite{bokhovkin2023mesh2tex} compress textures to attributes at the mesh facets, and differentiable rendering enables adversarial supervision in the image space.

Alternatively, Intrinsic Neural Fields (\textbf{INFs}) \cite{koestler2022intrinsic} introduces a latent representation wherein textures are encoded as eigenfunctions of the Laplacian taking values at the mesh vertices. However, features at points on faces are recovered through barycentric interpolation, which we will show leads to a breakdown in reconstruction fidelity at higher compression ratios. Recently, Manifold Diffusion Fields \cite{elhag2023manifold, zhuang2022diffusion} introduce the first fully intrinsic DMs defined over 2D Riemannian manifolds, and are based on a global attention mechanism which aggregates features at the vertices, limiting their applicability to coarse, low-resolution meshes due to complexity constraints. Here, our proposed FLDMs are also fully intrinsic but are built upon locally-supported convolutions defined over the manifold \cite{mitchel2021field}, allowing them to scale to higher-resolution meshes. 

Our single-textured-mesh FLDMs are inspired by recent work in which DMs are trained to capture internal patch distributions from a single image with the goal of generating diverse samples with similar visual content \cite{nikankin2022sinfusion, wang2022sindiffusion, kulikov2023sinddm}; In turn, these models can trace their lineage to the influential SinGAN architecture \cite{shaham2019singan} and its derivatives \cite{hinz2021improved, granot2022drop}. Here, we mirror the approach proposed by SinDiffusion \cite{wang2022sindiffusion} and denoise with a shallow convolutional UNet to control the receptive field and prevent the DM from overfitting to the texture. Similar to our setting, Sin3DM \cite{wu2023sin3dm} was the first to train LDMs in a single-textured-mesh paradigm, however they take an extrinsic approach to generating geometry \textit{and} texture, whereas our approach \textit{focuses only on synthesizing textures} which we show qualitatively are of higher-fidelity than those produced by Sin3DM.

%% file: sections/overview.tex
\section{Method Overview} 
\label{sec:overview}
We introduce a framework for intrinsic LDMs operating directly on surfaces. The first component is a latent representation of mesh textures comprised of tangent vector features. At each vertex, the tangent features characterize the texture of a local neighborhood on the surface, and the collection of these features constitute a stack of vector fields we call \textit{Field Latents} (\textbf{FLs}). The FL space is learned with a locally-supported variational autoencoder \cite{kingma2013auto} (\textbf{FL-VAE}) which is described in Section~\ref{sec:latents}.

The second key component is the \textit{Field Latent Diffusion Model} (\textbf{FLDM}), which learns to denoise a diffusion process in the tangent space of a surface. Denoising networks are constructed with field convolutions (\textbf{FCs}), surface convolution operators acting on tangent vector fields \cite{mitchel2021field}.  We extend FCs to interleave scalar embeddings with the tangent vector features, enabling the injection of diffusion time-steps and optional conditioning, such as user-specified labels, into the denoising model.

In practice, FL-VAEs and FLDMs are applied to synthesize new textures from a single textured triangle mesh.
We pre-train a single FL-VAE by superimposing planar meshes over a large high-quality image dataset, obtaining a general-purpose latent space which can be used to train FLDMs on arbitrary textured meshes. To do so, a texture is first mapped to distributions in the tangent space at the vertices with the FL-VAE encoder, and an FLDM is trained to iteratively denoise tangent vector latent features. Afterwards, samples from the FLDM can be decoded with the FL-VAE decoder to synthesize new textures.

\paragraph{Importance of Isometry-Equivariance} Our FL-VAEs and FLDMs are designed to be equivariant under isometries, \textit{i.e.} they commute with distance-preserving shape deformations. Unlike images, points on a surface have no canonical orientation. The choice of local coordinates is ambiguous up to an arbitrary rotation, making it impossible to simply adapt standard image-based VAEs and diffusion models to surfaces. However, isometries manifest locally as rotations. Thus, designing our framework to be isometry-equivariant inherently solves the orientation ambiguity problem, enabling consistent, repeatable results.  Isometry-equivariance also results in tangible  benefits --- our models are able to seamlessly reproduce textural details across locally similar areas of meshes. Furthermore, while existing models transfer textures with pointwise maps \cite{donati2022complex} or generatively via conditioning on a learned token \cite{richardson2023texture, chen2023shaddr}, we can simply sample our pre-trained FLDMs on new, similar meshes to texture them in the learned style.

%% file: sections/latents.tex
\section{Field Latents}
\label{sec:latents}

As in Knoppel \textit{et al.} \cite{knoppel2013globally} we associate tangent vectors with complex numbers.  For a surface $M$ and point $p \in M$, we assign to the tangent space an arbitrary orthonormal basis such that for any $\vn{v} \in T_pM$ we have $\vn{v} \equiv r e^{i\theta},$ with $r = \abs{\vn{v}}$ and $\theta$ the direction of $\vn{v}$ expressed in the frame. We make this convention explicit by denoting $\C_p \equiv T_pM$.

To describe the FL-VAE, we require a notion of multi-dimensional Gaussian noise in the tangent bundle of a surface $M$. Specifically, we denote the Gaussian distribution over $d$ copies of the tangent bundle $TM^d$ by \TND{M}. Samples $\epsilon \sim \TND{M}$ are  vector fields in $TM^d$ such that at each point $p \in M$, the coefficients of $\epsilon(p) \in \C_p^d$ expressed in the orthonormal basis are themselves samples from the $d-$dimensional standard normal distribution 
\begin{align}
    \epsilon(p) = \epsilon_1 + i \epsilon_2, \ \epsilon_1, \epsilon_2 \sim \mathcal{N}(0, I_d).\label{tangent_noise}
\end{align}

In this work we are concerned with isometries  $\diff:M \rightarrow N$, which enjoy the property that their push-forwards $d \diff: TM \rightarrow TN$, which define a smooth map between tangent spaces, manifest pointwise as rotations.  From Equation~(\ref{tangent_noise}) it is easy to see that $\TND{M}$ is symmetric under such transformations. Thus, it follows that sampling from \TND{N} is equivalent to pushing forward samples from \TND{M},
\begin{align} 
\epsilon' = \gamma \epsilon \qquad \textrm{for} \begin{aligned} \qquad \epsilon' \sim \TND{N} \\ \epsilon \sim \TND{M} \end{aligned} \label{equiv_noise}
\end{align}
with $\gamma \epsilon = \left[d\gamma \cdot \epsilon\right] \circ \diff^{-1}$ denoting the standard action of diffeomorphisms on vector fields via left-shifts.

\paragraph{FL-VAE Encoder}
The encoder $\enc$ is a network that, for any point $p \in M$,
takes $n-$dimensional scalar functions  $\psi \in \nice{M}{\R^n}$ (\textit{e.g.} texture RGB values with $n=3$) restricted to the surrounding geodesic neighborhood $\neigh{p} \subset M$ to the parameters of $d$ independent normal distributions in the tangent space at $p$:
\begin{align}
\begin{aligned}
\enc_{p}: \nice{\neigh{p}}{\mathbb{R}^n} &\rightarrow \C_p^d \times \mathbb{R}_{\geq 0}^d \\ %
\psi &\mapsto (\mu_p^{\psi}, \sigma_p^{\psi})
\end{aligned} \label{encoder}
\end{align}
Here, the means of the distributions $\mu_p^{\psi}$ are themselves \textit{tangent vectors} while the standard deviations $\sigma_p^{\psi}$ are \textit{scalars}. Latent codes characterizing the local restriction of $\psi$ are collections of tangent vectors $z_p^{\psi} \in \C_p^d$ drawn from the multivariate normal distribution parameterized by $(\mu_p^{\psi}, \sigma_p^{\psi})$. Using the reparameterization trick \cite{kingma2013auto} latent codes can be generated as $z_p^{\psi} = \mu_p^{\psi} + \sigma_p^{\psi} \odot \epsilon(p)$ with $\epsilon \sim \TND{M}$.  

\paragraph*{FL-VAE Decoder}
The principal advantage in a tangent vector latent representation is that it naturally admits a descriptive coordinate function which can be queried to associate features with neighboring points. The decoder \dec{} is designed as a neural field operating over the local parameterization of the surface induced by the logarithm map about a point $p$, $\log_p: \neigh{p} \rightarrow \C_p$,  with $\log_{p}q$ encoding the position of $q \in \neigh{p}$ in the tangent space of $p$. Formally,
\begin{align}
\begin{aligned}
\dec_{p}: \log_p\left(\neigh{p}\right) \times \C_p^d & \rightarrow \mathbb{R}^n \\ %
\left(\log_p q, z_p^{\psi}\right) &\mapsto \widehat{\psi}_p(q), 
\end{aligned} \label{decoder}
\end{align}
with $\widehat{\psi}_p$ denoting the prediction of $\psi$ made from the perspective of $p$. Decoding follows a similar process as proposed in \cite{gardner2022rotation}. Specifically, given a tuple $\left(\log_p q, z_p^{\psi}\right)$, the decoder constructs two features. First, an invariant scalar feature is obtained by vectorizing the upper-triangular part of the Hermitian outer product of the latent code with itself, $\textrm{vec}_{j \geq i} \left(z_p^{\psi} \left[z_p^{\psi}\right]^{*}\right) \in \mathbb{C}^{\frac{d(d+1)}{2}}$. Second, a positionally-aware scalar feature $c_{pq}^{\psi} \in \mathbb{C}^d$ is formed by the natural coordinate function 
\begin{align}
c_{pq}^{\psi} \equiv \log_pq \cdot \overline{z_p^{\psi}}, \label{direc_features}
\end{align}
which corresponds to storing the inner product and determinant of each tangent vector with the position of $q$ relative to $p$ given by $\log_{p}q$. The neural field receives the concatenation of these two features and returns the prediction $\widehat{\psi}_p(q)$.

\paragraph{Equivariance} Recall we are concerned with constructing a latent representation equivariant under isometries --- that for any $\psi \in \nice{M}{\R^n}$ and isometry $\diff: M \rightarrow N$ we have 
\begin{align}
\widehat{\left[\diff \psi\right]}_{\diff(p)} = \diff \widehat{\psi}_{p}, \label{fl_equiv} 
\end{align}
where  $\diff \psi = \psi \circ \diff^{-1}$ denotes the standard action of diffeomorphisms by left-shifts. 
Thus, it follows from Equation~(\ref{equiv_noise}) that a sufficient condition for Equation~(\ref{fl_equiv}) to hold is if for all $\psi \in \nice{M}{\R^n}$, isometries $\gamma: M \rightarrow N$, and $p \in M$, the encoded mean $\mu_p^{\psi}$ and standard deviation $\sigma_p^{\psi}$ satisfy
\begin{align}
\eval{d\diff}{p} \cdot \mu_p^{\psi} = \mu_{\diff(p)}^{\diff \psi} \qquad \textrm{and} \qquad \sigma_{p}^{\psi} = \sigma_{\diff(p)}^{\diff \psi} \label{fl_equiv_cond}, 
\end{align}
with $\eval{d\diff}{p}$ denoting the push-forward of $\diff$ in the tangent space at $p$, which rotates the local coordinate system. A detailed proof appears in the supplement, sec.~\ref{supp:fl_equiv}.

\paragraph*{Architecture} In practice, we discretize a surface $M$ by a triangle mesh with vertices $V$ and faces $F$. Textures are scalar functions $\psi$ defined continuously over the faces and vertices, taking values in $\R^3$. For each $p \in V$, we consider its neighborhood to be the surrounding one-ring, the collection of faces in $F$ that share $p$ as a vertex. The encoder architecture is chosen so as to satisfy the conditions in Equation~(\ref{fl_equiv_cond}). A fixed number of points $\{ q_i \}$ are uniformly sampled in the one-ring neighborhood at which the texture function is evaluated. The resulting scalar features $\{\psi(q_i)\}$ are interleaved with the logarithms $\{\log_pq_i\}$ to construct tangent vector features which are concatenated with a token following \cite{dosovitskiy2020image}. The features are passed to eight successive VN-Transformer layers \cite{assaad2022vn}, modified to process tangent vector features. Afterwards, the token is extracted and used to predict the mean  and standard deviation. 

The neural field comprising the decoder consists of a five-layer real-valued MLP, to which the real and imaginary parts of the complex feature vector are passed. For any point in a triangle, the latent codes at the adjacent vertices provide three distinct predictions which are linearly blended with barycentric interpolation only at inference. This is in contrast to INF \cite{koestler2022intrinsic}, wherein linearly blended features are passed to the neural field to make predictions.  A comprehensive discussion of the FL-VAE architecture and training objective is provided in section~\ref{fl_architecture} in the supplement.

%% file: sections/diffusion.tex
\section{Field Latent Diffusion Models}
\label{sec:diffusion}
FLDMs define a noising process on surface vector fields, and learn a denoising diffusion probabilistic model (\textbf{DDPM}) \cite{ho2020denoising} using an equivariant surface network. In principle they can be applied to any vector data on surfaces, and here we use them to learn feature distributions in FL space. 

Latent vector fields $Z \in TM^d$ can be produced by compressing a texture $\psi \in \nice{M}{\R^3}$ to a collection of tangent vector features at each point with the FL-VAE encoder as in Equation~(\ref{encoder}) and sampling from the distributions. The forward process progressively adds noise in the tangent space, dictated by a monotonically decreasing schedule $\{\alpha_t\}_{0 \leq t \leq T}$ with $\alpha_0 = 1$ and $\alpha_T = 0$. The noised vector fields can be expressed at an arbitrary time $t$ in the discretized forward process with 
\begin{align}
Z_t = \sqrt{\alpha_t}  \, Z + \sqrt{1 - \alpha_t} \,  \epsilon \in TM^d, \label{diff_forward}
\end{align}
where $\epsilon \sim \TND{M}$ as in Equation~(\ref{tangent_noise}).

The denoising function $\varepsilon$ is a surface network trained to reverse the forward process. Specifically, it is a learnable map on the space of latent vector fields, conditioned on timesteps $t$ and optional scalar embeddings $\rho \in \nice{M}{\R^m}$, with the goal of predicting the added noise,
\begin{align}
\begin{aligned}
\varepsilon: TM^d \times \R_{\geq 0} \times \nice{M}{\R^m} \rightarrow TM^d,
\end{aligned} \label{denoising_network}
\end{align}
trained subject to the loss 
\begin{align}
L_{\textrm{FLDM}} = \E_{\epsilon \sim \TND{M},0 \leq t \leq T} \left\lVert \epsilon - \varepsilon(Z_t, t, \rho) \right\rVert_2^2. 
\label{FLDM_loss}
\end{align}
 
At inference, new latent vector fields are generated following the iterative DDPM denoising process.  Beginning with representing the fully-noised vector fields at step $T$ by a sample $\widetilde{Z} \sim \TND{M}$, $\widetilde{Z}_T = \widetilde{Z}$, the trained denoising function $\varepsilon$ is used to estimate the previous step in the forward process via the relation
\begin{align}
\begin{aligned}
\widetilde{Z}_{t-1} & = C_1(\alpha_t, \alpha_{t-1}) \, \widetilde{Z}_t + C_2(\alpha_t, \alpha_{t-1}) \, \varepsilon(\widetilde{Z}_t, t, \rho)  \\
& + C_3(\alpha_t, \alpha_{t-1}) \, \epsilon,
\end{aligned} \label{reverse_relation}
\end{align}
with $\epsilon = 0$ for $t = 1$ and $\epsilon \sim \TND{M}$ otherwise, and $C_i(\alpha_t, \alpha_{t-1})\in \mathbb{R}$ as defined in \ref{fl_architecture} of the supplement. 

\paragraph{Equivariance}

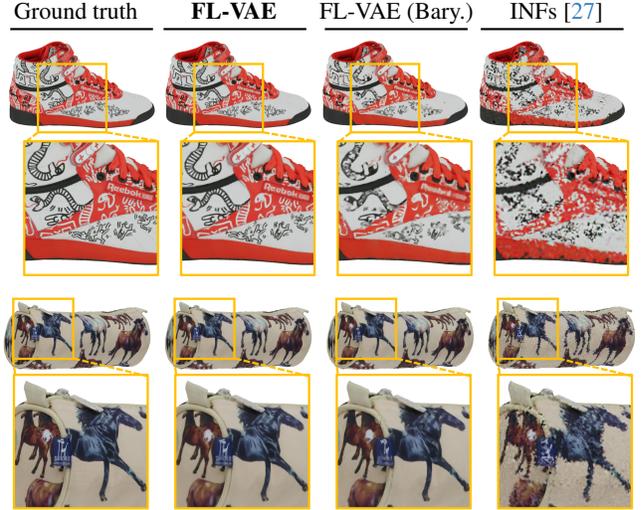
\begin{figure}[t]
\begin{picture}(\linewidth,0.81\columnwidth)
\put(-7,-5){{\includegraphics[width=1.05\linewidth]{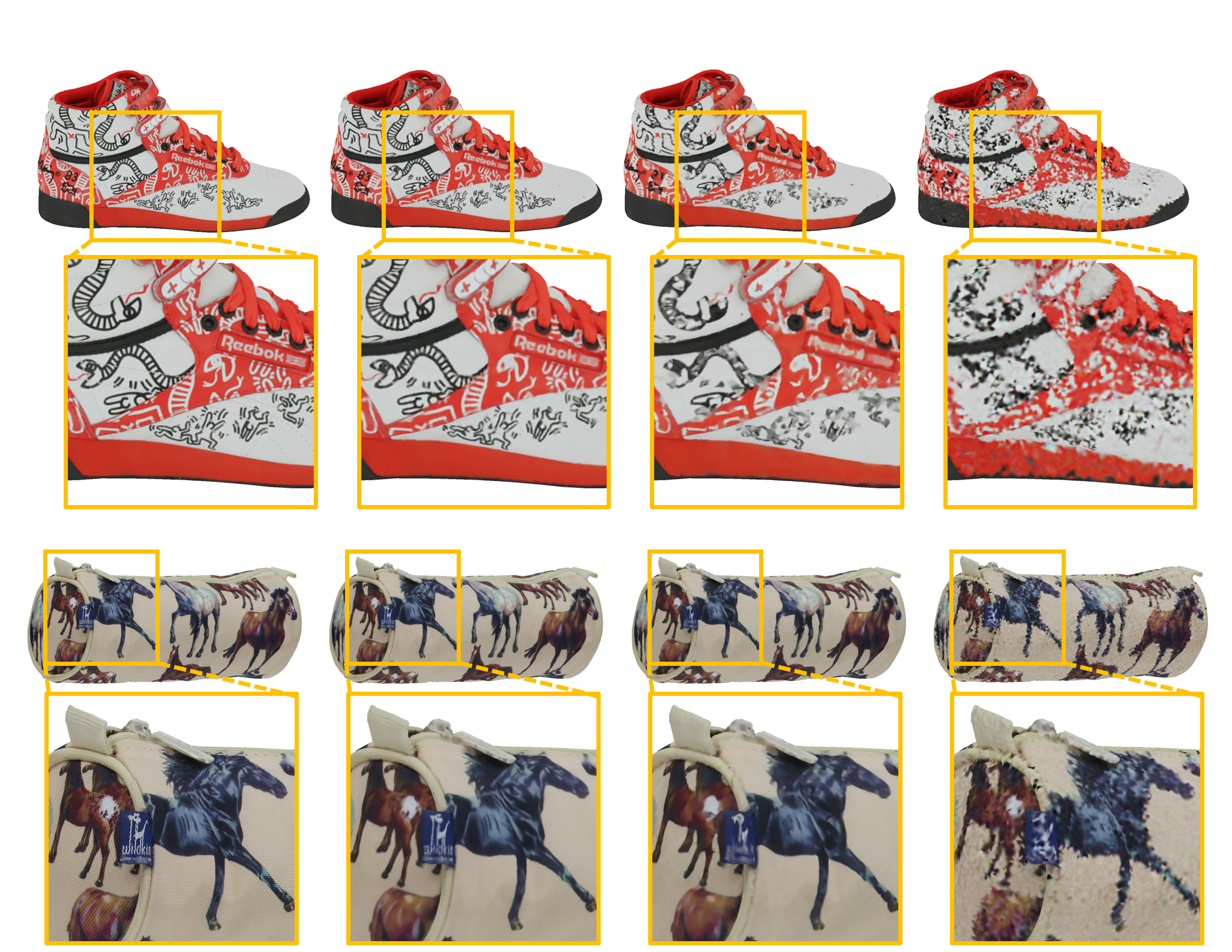}}}
\thicklines
\put(0, 182){ {\small Ground truth}}
\put(1, 178){\line(1,0){54}} 
\put(67, 182){ {\small \textbf{FL-VAE}}}
\put(59, 178){\line(1,0){54}} 
\put(118, 182){{\small FL-VAE (Bary.)}}
\put(120, 178){\line(1,0){54}} 
\put(190, 182){{\small INFs \cite{koestler2022intrinsic}}}
\put(179, 178){\line(1,0){54}} 
\end{picture}
\caption{
Visual comparison of textures compressed and reconstructed with the FL-VAE and INF \cite{koestler2022intrinsic} on $30$K vertex meshes. Compared to barycentric coordinates, our proposed logarithmic coordinate function more richly extends latent features across the mesh, enabling the reconstruction of finer details. \textit{Zoom in to view.}}
\vspace{-.2in}
\label{tex_recon_fig}
\end{figure}

\begin{table}
\centering
\resizebox{0.85\columnwidth}{!}{
    \begin{tabular}{lccc}
    Method & PSNR $\uparrow$ & DSSIM $\downarrow$ & LPIPS $\downarrow$ \\ \hline \hline
    \multicolumn{1}{l|}{\textbf{FL-VAE}}         & \multicolumn{1}{c|}{22.38 (20.59)} & \multicolumn{1}{c|}{0.51 (0.83)} & \multicolumn{1}{c}{1.02 (1.81)}   \\
    \multicolumn{1}{l|}{FL-VAE (Bary.)} & \multicolumn{1}{c|}{21.33 (19.61)} & \multicolumn{1}{c|}{0.66 (1.01)} & \multicolumn{1}{c}{1.31 (2.03)}  \\
    \hline  
    \multicolumn{1}{l|}{INFs \cite{koestler2022intrinsic}} & \multicolumn{1}{c|}{18.86 (16.45)} & \multicolumn{1}{c|}{1.16 (1.64)} & \multicolumn{1}{c}{2.15 (2.73)}
    \end{tabular}
}
\vspace{-.1in}
\caption{Texture reconstruction quality on meshes from the Google Scanned Objects dataset \cite{downs2022google}. Reconstructions are evaluated on high-res ($30$K vertices) and low-res ($5$K vertices, in parenthesis) meshings of the objects.  Metrics are computed in the texture atlases, with DSSIM and LPIPS scaled by a factor of $10$.}
\label{recon_results_table}
\vspace{-.2in}
\end{table}

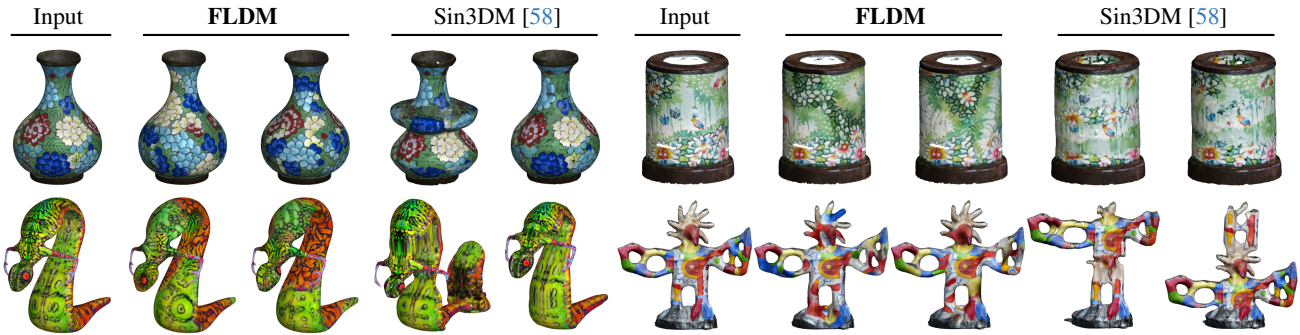
\begin{figure*}[!ht]
\begin{picture}(\linewidth,0.5\columnwidth)
\put(0,0){{\includegraphics[width=\linewidth]{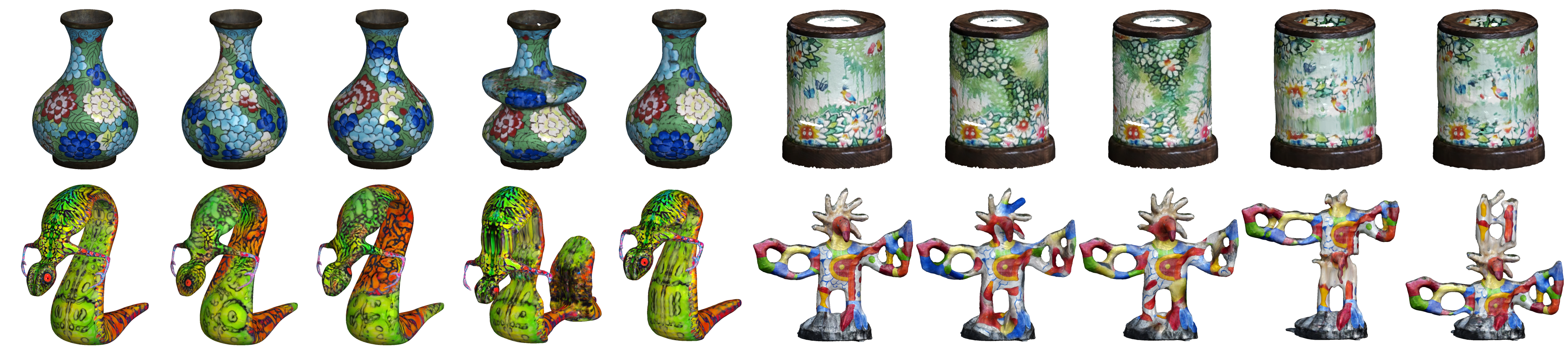}}}
\thicklines
\put(15, 118){ {\small Input}}
\put(9, 114){\line(1,0){38}} 
\put(81, 118){ {\small \textbf{FLDM}}}
\put(59, 114){\line(1,0){78}} 
\put(169, 118){{\small Sin3DM \cite{wu2023sin3dm}}}
\put(153, 114){\line(1,0){78}} 
\put(252, 118){ {\small Input}}
\put(245, 114){\line(1,0){40}} 
\put(326, 118){ {\small \textbf{FLDM}}}
\put(302, 114){\line(1,0){82}} 
\put(421, 118){{\small Sin3DM \cite{wu2023sin3dm}}}
\put(405, 114){\line(1,0){82}} 
\end{picture}
\caption{ Visual comparison of unconditionally (label-free) generated textures from FLDMs and Sin3DM \cite{wu2023sin3dm}.  Our FLDM's isometry-equivariant construction allows for replication of textural details across locally similar regions. In contrast, Sin3DM's extrinsic approach associates textural features with 3D space; Synthesized details appear as repetitions or extrusions of previous patterns along the major axes of the mesh, and we observe that novel textures cannot be created without modifying geometry. \textit{Zoom in to compare.} }
\label{unc_gen_fig}
\end{figure*}

From the equivalence of the distributions $\TND{M}$ and $\TND{N}$ under isometries as in Equation~(\ref{equiv_noise}), it follows that for any isometry $\gamma: M \rightarrow N$,  a sufficient condition for \textit{both} the forward and reverse processes in Equations~(\ref{diff_forward})~and~(\ref{reverse_relation}) to satisfy 
\begin{align} 
\diff Z_t = \left[\diff Z\right]_t \qquad \textrm{and} \qquad \diff \widetilde{Z}_t = \big[\diff \widetilde{Z}\big]_t, \label{diff_equiv}
\end{align}
for $0 \leq t \leq T$ is if the denoising network $\varepsilon$ in Equation~(\ref{denoising_network}) has the property that for any vector field $Z \in TM^d$,
\begin{align}
\diff \varepsilon(Z, t, \rho) = \varepsilon( \gamma Z, t, \gamma \rho). \label{diff_equiv_cond}
\end{align}
A detailed proof is provided in sec.~\ref{supp:fldm_equiv} of the supplement.

\paragraph{Denoising with Field Convolutions}
To satisfy the condition in Equation~(\ref{diff_equiv}), our denoising network is constructed with field convolutions (FC)~\cite{mitchel2021field}, which belong to a class of models that can operate on tangent vector fields ~
\cite{Wiersma2020, wiersma2022deltaconv}.

FCs convolve vector fields $Z \in TM^C$ with $C' \times C$ filter banks $\vn{f}_{c'c} \in \nice{\C}{\C}$, returning vector fields $Z * \vn{f} \in TM^{C'}$. In conventional diffusion models, features inside the denoising network are conditioned on timesteps $t$ via an additive embedding applied inside the convolution blocks. Unfortunately, no direct analogue exists for vector field features on surfaces --- additive embeddings break equivariance and we find that multiplicative embeddings destabilize training. Thus, to enable conditioning of our denoising network $\varepsilon$ on both timesteps $t$ and optional user-input features $\rho \in \nice{M}{\R^m}$, we inject embeddings directly into convolutions by extending FCs to convolve vector fields with filters $\vn{f}_{c'c} \in \nice{\C \times \R^e}{\C}$, with $e$ the embedding dimension.  Critically, this formulation preserves equivariance without sacrificing stability. 

Following \cite{wang2022sindiffusion}, the denoising network $\varepsilon$ takes the form of a shallow two-level U-Net but with field convolutions; the relatively small size of the network's receptive field allows distributions of latent features to be learned without overfitting to the single textural example \cite{wu2023sin3dm, kulikov2023sinddm, wang2022sindiffusion}.  

%% file: sections/experiments.tex
\section{Experiments}
\label{sec:experiments}

\begin{table}{
    \centering
    \begin{picture}(\linewidth,0.25\columnwidth)
    \thicklines
    \put(0, 33){
    \resizebox{0.55\columnwidth}{!}
    {    
    \begin{tabular}{lcc}
    Method & SIFID $\downarrow$ & LPIPS $\uparrow$  \\ \hline \hline
    \multicolumn{1}{l|}{\textbf{FLDM}}         & \multicolumn{1}{c|}{3.27} & \multicolumn{1}{c}{1.15}   \\ 
    \multicolumn{1}{l|}{Sample (left)}         & \multicolumn{1}{c|}{0.71} & \multicolumn{1}{c}{0.94} \\
    \hline
    \multicolumn{1}{l|}{Sin3DM \cite{wu2023sin3dm}} & \multicolumn{1}{c|}{6.58} & \multicolumn{1}{c}{2.20} \\
    \multicolumn{1}{l|}{Sample (right)} & \multicolumn{1}{c|}{2.42} & \multicolumn{1}{c}{2.91}
    \end{tabular}
    }}
    \put(145,0){{\includegraphics[width=0.35\linewidth]{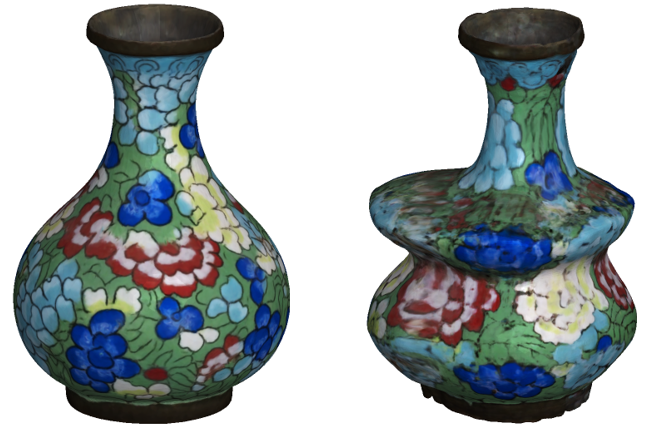}}}
    \put(149, 60){ {\small \textbf{FLDM}}}
    \put(148, 57){\line(1,0){34}} 
    \put(192, 60){ {\small Sin3DM }}
    \put(193, 57){\line(1,0){34}} 
    \end{picture}
    \caption{Fidelity and diversity of unconditionally generated textures trained on selected Objaverse \cite{deitke2022objaverse} and Scanned Objects \cite{downs2022google} meshes. SIFID and LPIPS are scaled by factors of $10^5$ and $10$, respectively. Unlike our FLDMs, Sin3DM synthesizes new geometry which  can inflate diversity scores. As an example, the individual metrics are shown for the FLDM and Sin3DM samples on the right. The latter attains a high LPIPS score, but exhibits poor textural diversity and quality. }
    \label{generation_results_table}
    } 
    \vspace{-.21in}
\end{table}

%


\paragraph*{Pre-training the FL-VAE} It is important to note that the FL-VAE proposed in Section~\ref{sec:latents} only interfaces with the mesh geometry through its local flattening in the image of the logarithm map. Thus, the FL-VAE operates independently of the local curvature, and a model pre-trained' on flat textured meshes can be used to encode and decode textures on arbitrary 3D shapes. 

In all of our experiments we use the same FL-VAE with a $d=8$ dimensional latent space, pre-trained on the OpenImagesV4 dataset \cite{kuznetsova2020open}. Specifically, we generate $1$K different planar meshes containing $10$K vertices each.
Each sample $512\times512$ image is overlaid by a randomly chosen triangulation, with the image serving as the ``texture'' defined over the mesh.    The FL-VAE is then trained with the sum of reconstruction loss and  KL divergence using the Adam \cite{kingma2017adam} optimizer. We deliberately train the FL-VAE on coarser triangulations than we expect during inference on textured meshes. This forces the model to reconstruct high frequencies during training, which we find increases robustness to sampling density and triangulation quality, in addition to improving the fidelity of reconstructed FLDM samples. 

\paragraph*{Training single-textured-mesh FLDMs}
All diffusion experiments share the same training regime.
During preprocessing, we create 500 copies of a given mesh each with 30K vertices and different triangulations to prevent the denoising network from learning the connectivity. Each copy shares the same UV-atlas as the original and the texture is mapped to vector field latent features at the vertices using the pre-trained FL-VAE encoder. FLDMs have a max timestep of $T = 1000$ and noise is added on a linear schedule. The denoising network is trained subject to the loss in Equation~(\ref{FLDM_loss}) using the Adam \cite{kingma2017adam} optimizer. 

\subsection{Texture Compression and Reconstruction}
\label{sec:eval_recon}

First we quantitatively evaluate the capacity of our proposed FL-VAE for compressing and reconstructing textures. We compare FL-VAEs against two other approaches: the state-of-the-art INF \cite{koestler2022intrinsic} and a modified version of FL-VAEs serving as an ablation. INFs train a neural field per texture and predicts the texture values using the values of the Laplacian eigenfunctions extended over triangles through barycentric interpolation. Similarly, the modified FL-VAE (denoted FL-VAE Barycentric) omits our proposed logarithmic coordinate function in Equation~(\ref{direc_features}) in favor of linearly interpolating the values of the invariant outer product features across triangles. It is otherwise identical to our proposed approach, including pre-training methodology. Thus, this also serves to compare the efficacy of our proposed coordinate function against barycentric interpolation in extending vertex features continuously over the mesh.

To evaluate, we select $16$ objects exhibiting a range of complex textures with high-frequency details from the Google Scanned Objects dataset~\cite{downs2022google}. Two versions of the dataset are used, with objects remeshed to have $30$K (high-res) and $5$K (low-res) vertices, respectively. The single pre-trained FL-VAE is used to map the textures to distributions over the tangent space, and the mean is used to sample the reconstructed texture at the pixel indices in the texture atlas. A separate INF neural field is trained for each mesh, using the specified regime for high-res reconstruction \cite{koestler2022intrinsic}.

Following INFs, we evaluate the quality of the reconstructed textures using the average PSNR, DSSIM \cite{wang2004image}, and LPIPS \cite{zhang2018unreasonable} metrics computed in texture atlases. The results are reported in Table~\ref{recon_results_table}, with several examples of reconstructions shown in Figure~\ref{tex_recon_fig}.  Our proposed FL-VAE achieves the best results across all metrics at both resolutions, followed by the barycentric FL-VAE and INFs. The significant gap in performance between the barycentric FL-VAE and INFs is likely due to the reliance of the latter the on the Laplacian eigenfunctions as input features. Areas of fine textural detail may not necessarily overlap with regions where the eigenfunctions exhibit high variability, potentially contributing to the ``raggedness'' observed in INF reconstructions.  Furthermore, the superior performance of our proposed FL-VAE relative to its barycentric counterpart suggests that our logarithmic coordinate function in Equation~(\ref{direc_features}) provides a richer interpolant over the mesh faces. This directly increases the representational capacity of the latent features, enabling the reconstruction of finer textural details with the same number of features.

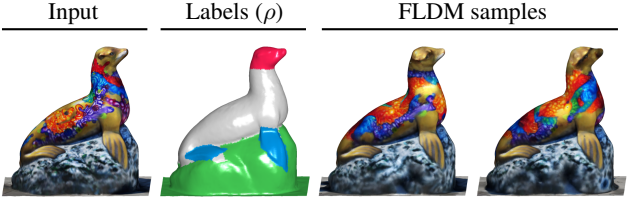
\begin{figure}[t]
\begin{picture}(\linewidth,0.28\columnwidth)
\put(0,0){{\includegraphics[width=\linewidth]{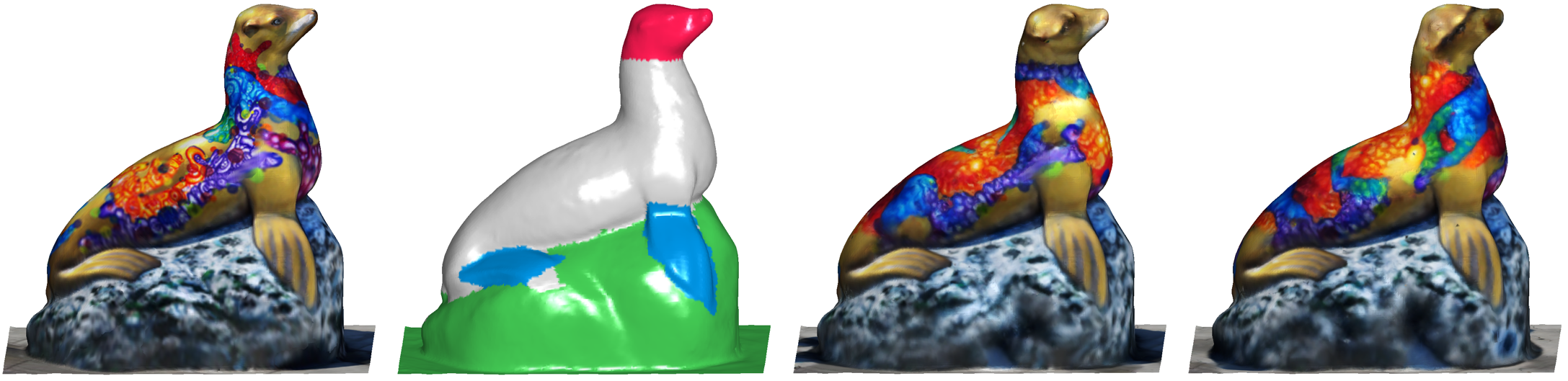}}}
\thicklines
\put(15, 68){ {\small Input}}
\put(0, 64){\line(1,0){55}} 
\put(67, 68){ {\small Labels ($\rho$)}}\put(60, 64){\line(1,0){55}}
\put(150, 68){{\small FLDM samples}}
\put(121, 64){\line(1,0){112}}
\end{picture}
\caption{Label-guided generation. FLDMs can be conditioned on a subjective user-input labeling which generated textures will reflect. See Figure~\ref{teaser_fig} for more examples. \textit{Zoom in to view}. }
\vspace{-.2in}
\label{label_guided_fig}
\end{figure}

\subsection{Unconditional (Label-Free) Generation}
\label{sec:unc_gen}
Label-free generation can be performed by training and sampling FLDMs conditioned only on the diffusion timescale (\textit{i.e.} $\rho$ = 0). We compare against Sin3DM \cite{wu2023sin3dm} which to our knowledge is the only prior generative model designed to operate in a single-textured-mesh paradigm. We follow their proposed evaluation protocol, and train FLDMs and Sin3DM on $10$ assets from the Objaverse \cite{deitke2022objaverse} and Scanned Objects \cite{downs2022google} datasets selected for texture complexity and diversity. For each asset, $64$ textures are sampled from the trained models, which are rendered on the mesh from $24$ different views. Fidelity is measured with the mean SIFID (Single Image Fr\'echet Inception Distance) \cite{shaham2019singan} pairwise between the renderings and those of the input mesh from the same view; the mean LPIPS \cite{zhang2018unreasonable} between renders of samples from the same viewpoint measures diversity. We note that unlike FLDMs, which only synthesize new textures, Sin3DM synthesizes both textures \textit{and} geometries, with the latter potentially inflating LPIPS scores.

The results are shown in Table~\ref{generation_results_table}, Figure~\ref{unc_gen_fig}, and in the supplement (Figure~\ref{unc_gen_supp_fig}).
Compared to Sin3DM, our FLDMs are able to synthesize higher-fidelity textures which is likely due in part to the former's rasterization of textured geometry to a 3D grid before encoding to the latent space, potentially aliasing fine details. More generally, we observe that by operating over an \textit{extrinsic} representation of textured geometry, Sin3DM intertwines texture with a mesh's 3D embedding. Thus, new textural details appear as repetitions or extrusions of earlier patterns along the major axes of the models, as seen in the vase, snake, and sculpture samples in Figure~\ref{unc_gen_fig}. Furthermore, we observe that texture and position are so strongly linked that Sin3DM is unable to produce significant textural alterations without correspondingly modifying the geometry, as seen in the vase, snake, and lantern samples. Here, the advantage of our \textit{isometry-equivariant} formulation becomes clear as our FLDMs can effectively replicate textural details across areas of the mesh that are locally similar, regardless of the relative configuration of these regions in the 3D embedding.

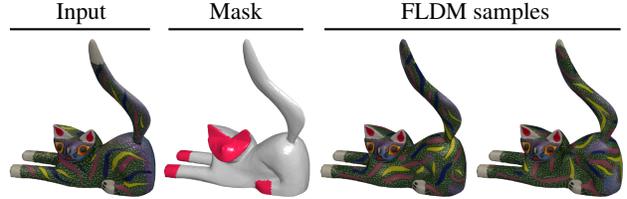
\begin{figure}[t]
\begin{picture}(\linewidth,0.28\columnwidth)
\put(0,0){{\includegraphics[width=\linewidth]{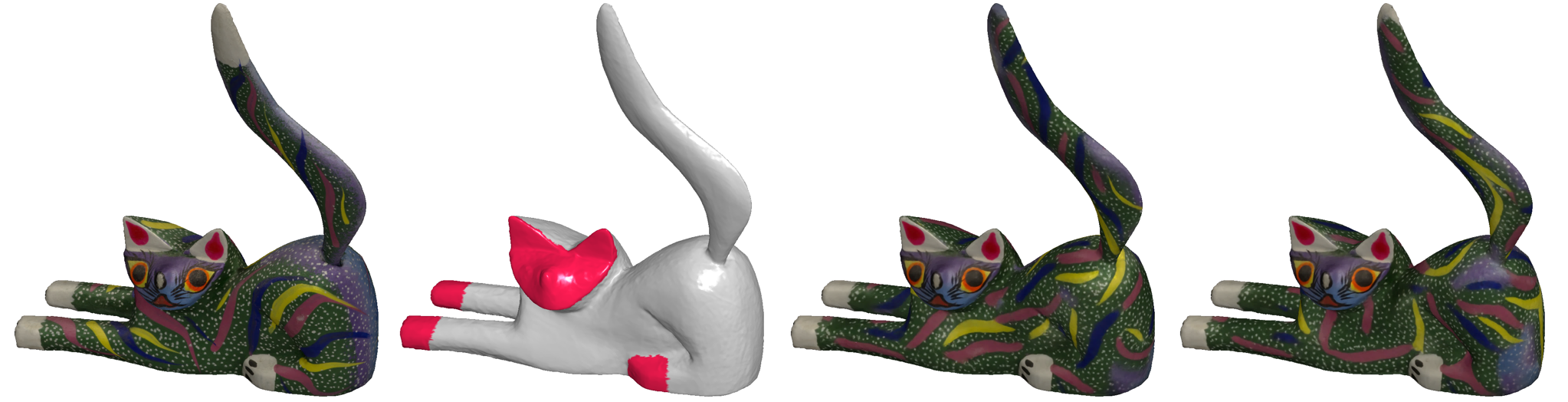}}}
\thicklines
\put(17, 68){ {\small Input}}
\put(2, 64){\line(1,0){53}} 
\put(75, 68){ {\small Mask}}
\put(62, 64){\line(1,0){51}} 
\put(150, 68){{\small FLDM samples}}
\put(121, 64){\line(1,0){112}}
\end{picture}
  \caption{Inpainting with FLDMs. The input texture is preserved in the masked regions and new content is synthesized elsewhere with agreement on the boundaries. \textit{Zoom in to view}.}
\vspace{-.20in}
\label{inpaint_fig}
\end{figure}

\subsection{Label-Guided Generation}
\label{sec:label_guided}

Textured objects often exhibit distinct content specific to certain regions of the mesh which may not be geometrically unique; for example, in Figure~\ref{teaser_fig} regions on the sole, body, and inside of the shoe lack distinguishing geometric features, as is the case with the eyes and mouth of the skull and the decal on the bottle. In such cases, a user may find it desirable for synthesized textures to reflect a subjective distribution of content on the input mesh, which we facilitate by extending our denoising model to incorporate optional conditioning from user-designated mesh features $\rho \in \nice{M}{\R^m}$. In the simplest case, the user can paint a coarse, semantic labeling to subjectively delineate salient regions, such as the flippers, head, and rock base of the seal sculpture in Figure~\ref{label_guided_fig}. After training an FLDM conditioned on the labels, generated textures reflect the specified segmentation. Further examples are shown in Figure~\ref{teaser_fig}.  Labeled areas contain different textural features, and conditioning ensures the FLDM samples respect these distributions, \textit{e.g.} that teeth are synthesized only on the mouth.

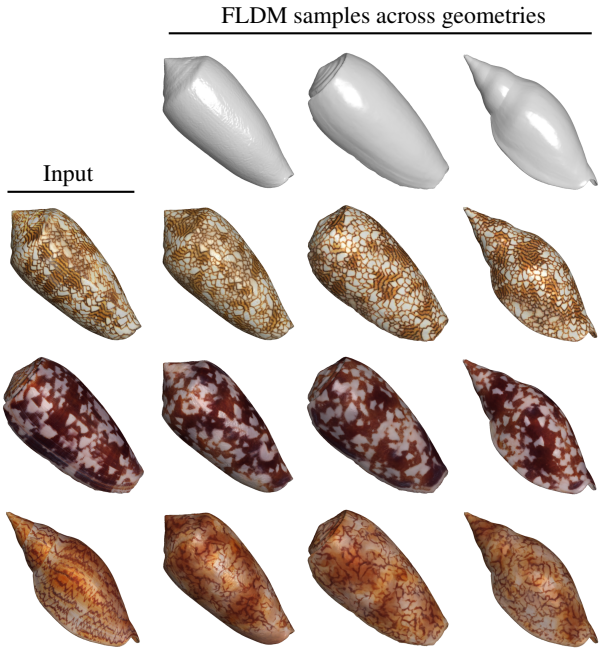
\begin{figure}[t]
\begin{picture}(\linewidth,1.05\columnwidth)
\put(0,0){{\includegraphics[width=\linewidth]{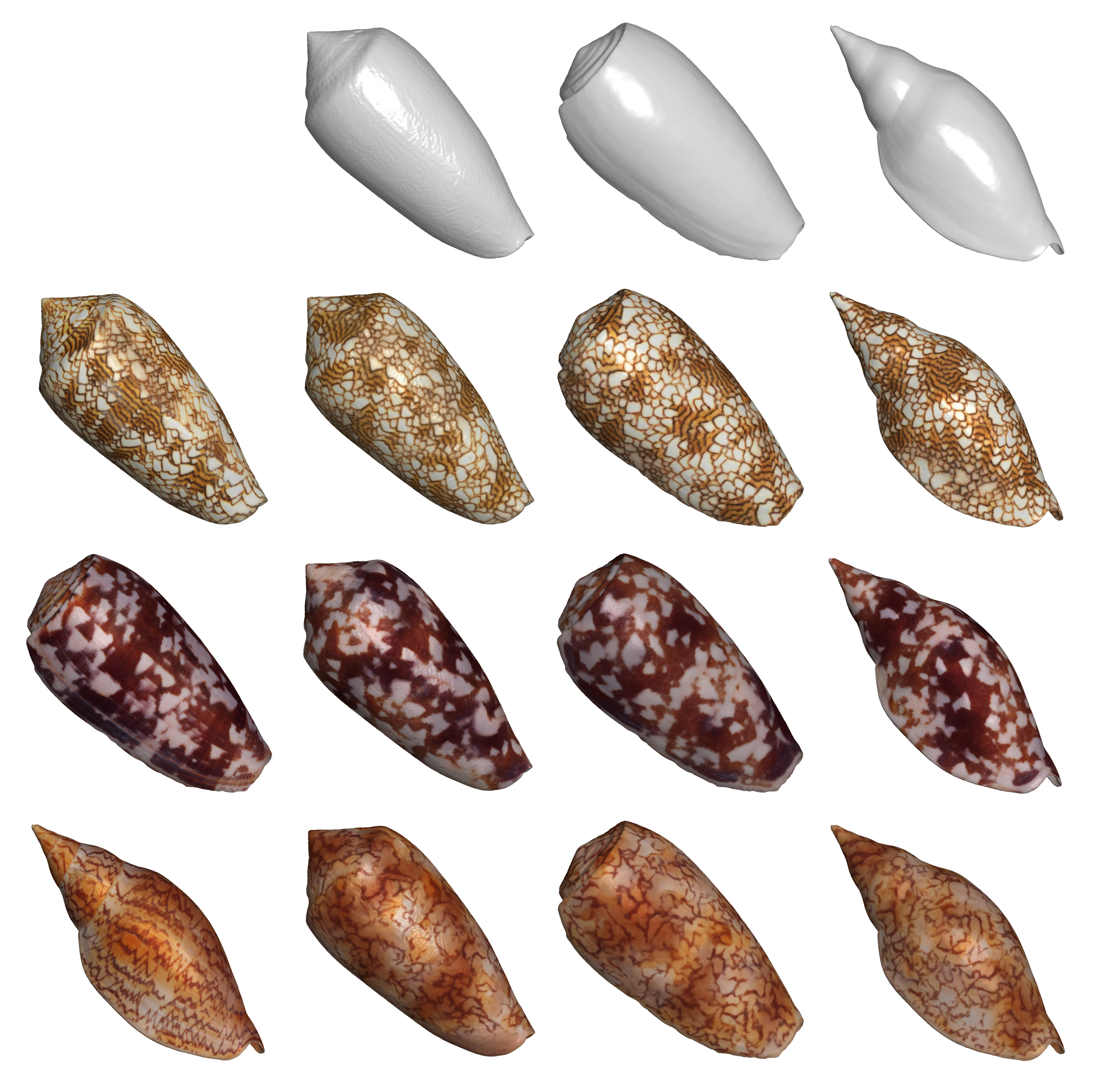}}}
\thicklines
\put(18, 183){ {\small Input}}
\put(7, 179){\line(1,0){48}} 
\put(88, 243){{\small FLDM samples across geometries}}
\put(68, 239){\line(1,0){160}} 
\end{picture}
\caption{Generative texture transfer. Since our FL-VAE and FLDMs commute with isometries, we can encode a texture and train an FLDM on an input mesh (left), then sample the pre-trained FLDM on a new, similar mesh and decode the latents to texture it (center to right). Conditioning labels segmenting the shell apertures are not visible. \textit{Zoom in to view.}}
\vspace{-.2in}
\label{transfer_fig}
\end{figure}

\subsection{Inpainting}
In certain cases a user may wish to preserve the input texture in certain regions while synthesizing new content elsewhere.   Texture inpainting is a well studied topic \cite{fayer2018texturing, grigorev2019coordinate, maggiordomo2023texture}, and here we propose a straight-forward approach with FLDMs. Given a pre-trained FLDM conditioned on a user-input binary mask $m: M \rightarrow \{0, 1\}$ specifying the regions to preserve (Figure~\ref{inpaint_fig}), generative inpainting can be performed by modifying the iterative sampling process. Specifically, after estimating previous the latent features $\widetilde{Z}_{t-1}$ as in Equation~(\ref{reverse_relation}), the features in the mask region are replaced with the appropriately noised input latents from the forward process $Z_{t-1}$ such that
\begin{align}
\widetilde{Z}_{t-1} \mapsto m \cdot Z_{t-1} + (1 - m) \cdot \widetilde{Z}_{t-1}. \label{inpaint_step}
\end{align}
This approach encourages the inpainted texture to agree with the original at the mask boundaries due to convolutional structure of the denoising network, with this effect observed on the hind leg of sampled textures in Figure~\ref{inpaint_fig}.

\subsection{Generative Texture Transfer}
\begin{figure}[t]
\begin{picture}(\linewidth,0.27\columnwidth)
\put(0,0){{\includegraphics[width=\linewidth]{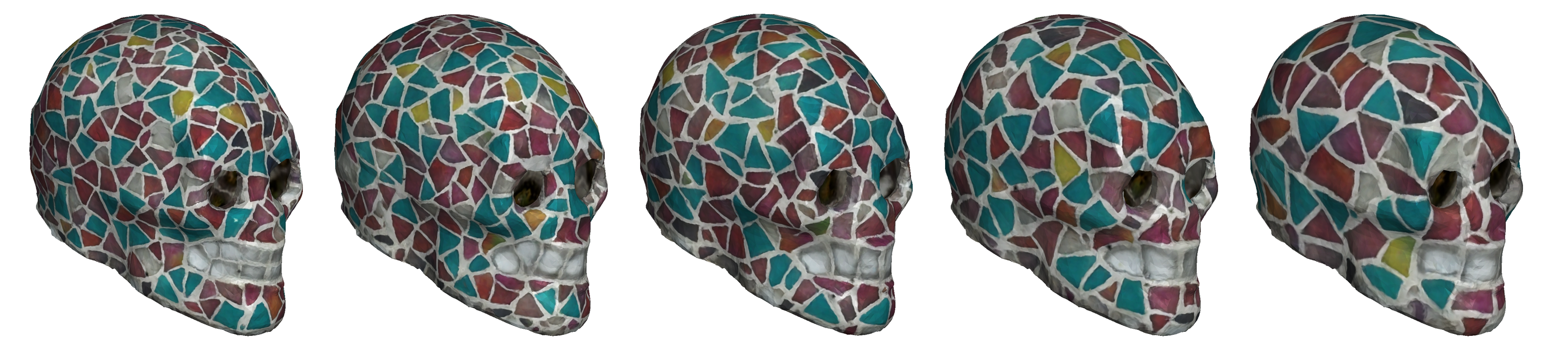}}}
\thicklines
\put(22, 60){ {\small FLDM samples on increasingly coarse remeshings}}
\put(10, 56){\vector(1,0){215}} 
\end{picture}
\caption{ Controlling the scale of synthesized textures.  Here, the FLDM trained on the textured skull in the center-left of  Figure~\ref{teaser_fig} is sampled on progressively coarser remeshings of a new skull geometry, dilating the size of the tesserae. \textit{Zoom in to view details.} }
\label{scale_fig}
\vspace{-.2in}
\end{figure}
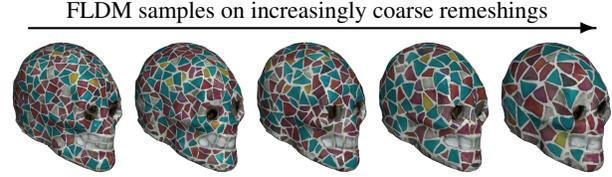

Perhaps the most dramatic consequence of isometry-equivariance is that it naturally facilitates a notion of generative texture transfer.  Since the FL-VAE and FLDMs commute with isometries $\diff: M \rightarrow N$ as in Equations~(\ref{fl_equiv})~and~(\ref{diff_equiv}), our models see $TM^d$ and $TN^d$ as functionally the same space. In practice, we find that long as surfaces $M$ and $N$ are \textit{approximately} isometric, we can encode a texture and train an FLDM on $M$, then sample the model on $N$ and decode the synthesized latents to texture it.  

Several examples are shown in Figure~\ref{transfer_fig}. An FLDM is pre-trained with label conditioning on each of the textured shells in the first column and sampled with corresponding labels on the others to texture them in the same style. Additional examples are shown on the right side of Figure~\ref{teaser_fig}. Notably, our model successfully transfers between different topologies, as the skull meshes are genus zero and three. More generally, we are able to transfer textures between highly dissimilar shapes (see supplement for examples).

\subsection{Controlling Textural Scale}
An additional feature of our FLDMs is that they offer user-control over the scale of sampled textures. Internally, FLDMs normalize the filter support by dividing the distances between adjacent vertices by the radius of the mean vertex area element $\widetilde{r} = \sqrt{A / (\pi \abs{V})}$, where $A$ is the surface area of the mesh and $\abs{V}$ is the number of vertices. Thus, the scale of the synthesized textural details can be controlled by increasing or decreasing the resolution of the mesh upon which the pre-trained FLDM is sampled.  An example shown in Figure~\ref{scale_fig}, where the FLDM trained on the textured skull on the center left of Figure~\ref{teaser_fig} is sampled on progressively coarser remeshings of a new skull geometry.  As the number of vertices decreases, the average one-ring occupies an increasingly larger fraction of the surface area, dilating the size of the synthesized textural details.

%% file: sections/conclusion.tex
\section{Conclusion and Limitations}
We present an original framework for latent diffusion models operating directly on the surfaces to generate high-quality textures. Our approach consists of two main contributions: a novel latent representation mapping textures to vector fields called \textit{Field Latents} and \textit{Field Latent Diffusion Models} which learn to denoise a diffusion process in the latent space on the surface. We apply our method in a \textit{single-textured-mesh paradigm}, and generate variations of textures with state-of-the-art fidelity. Limitations of our approach include the relatively long training time of FLDMs and an inability to reflect directional information in synthesized textures. The latter could potentially be addressed by conditioning FLDMs on user-specified vector fields, with embeddings formed via dot products with network features.

%% file: supplement/equivariance.tex
\section{Equivariance}
Here we provide detailed proofs of the claims made in Sections~\ref{sec:latents}~and~\ref{sec:diffusion} regarding sufficient conditions for isometry-equivariance. 

\subsection{Field Latents}
\label{supp:fl_equiv}
Consider the FL variational autoencoder consisting of the encoder $\enc$ and decoder $\dec$ defined pointwise as in Equations~(\ref{encoder}) and (\ref{decoder}), respectively. 

\begin{claim}
If for all $\psi \in \nice{M}{\R^n}$, isometries $\diff: M \rightarrow N$, and points $p \in M$, the encoded mean and standard deviation satisfy the condition of Equation~(\ref{fl_equiv_cond})
\begin{align}
\eval{d\diff}{p} \cdot \mu_p^{\psi} = \mu_{\diff(p)}^{\diff \psi} \qquad \textrm{and} \qquad \sigma_{p}^{\psi} = \sigma_{\diff(p)}^{\diff \psi}, \nonumber 
\end{align}
then the FL-VAE commutes with isometries such that 
\begin{align}
\widehat{\left[\diff \psi\right]}_{\diff(p)} = \diff \widehat{\psi}_{p}. \nonumber
\end{align}
\end{claim}

\begin{proof}
First, we observe that for any isometry $\diff: M \rightarrow N$ and point $p \in M$, the differential $\eval{d\diff}{p}$ is an orthogonal transformation taking vectors in $T_pM$ to $T_{\diff(p)}N$. Thus, within our complexification of the tangent space, $\eval{d\diff}{p}$ is an element of $\textrm{U}(1)$, the group of complex numbers with unit modulus, acting multiplicatively on $\C_p$.

Now, consider any $\psi \in \nice{M}{\R^n}$, isometry $\diff: M \rightarrow N$, point $p \in M$, and suppose the encoded mean $\mu_p^{\psi}$ and standard deviation $\sigma_{p}^{\psi}$ satisfy the condition. Then, we can relate latent codes distributed in $T_{\diff(p)}N$ to those in $T_{p}M$ with 
\begin{align}
z_{\diff(p)}^{\diff \psi} & \stackrel{\phantom{(\ref{equiv_noise})}}{=} \mu_{\diff(p)}^{\diff \psi} + \sigma_{\diff(p)}^{\diff \psi} \odot \epsilon'(\diff(p)), \quad \epsilon' \sim \TND{N}  \nonumber \\ 
&\stackrel{(\ref{equiv_noise})}{=}\mu_{\diff(p)}^{\diff \psi} + \eval{d \diff}{p} \cdot \sigma_{\diff(p)}^{\diff \psi} \odot \epsilon(p), \quad \epsilon \sim \TND{M} \nonumber \\
& \stackrel{(\ref{fl_equiv_cond})}{=} \eval{d \diff}{p} \cdot \mu_{p}^{\psi} + \eval{d \diff}{p} \cdot \sigma_{p}^{\psi} \odot \epsilon(p) \nonumber \\
&\stackrel{\phantom{(\ref{fl_equiv_cond})}}{=} \eval{d \diff}{p} \cdot z_{p}^{\psi} \label{latent_vector_equiv}
\end{align}
Using the relationship between the latent codes and the fact that $\log_{\diff(p)}\diff(q) = \eval{d \diff}{p} \cdot \log_p q$ for $q \in \neigh{p}$ as long as $\diff$ is an isometry, it follows that the coordinate function in Equation~(\ref{direc_features}) is invariant with 
\begin{align}
c_{\diff(p)\diff(q)}^{\diff \psi} & \stackrel{(\ref{direc_features})}{=} \log_{\diff(p)}\diff(q) \cdot \overline{z_{\diff(p)}^{\diff \psi}} \nonumber \\ 
& \stackrel{(\ref{latent_vector_equiv})}{=} \eval{d\diff}{p} \cdot \log_{p}q \cdot \overline{\eval{d \diff}{p}} \cdot \overline{z_{p}^{\psi}} \nonumber \\
& \stackrel{(\ref{direc_features})}{=} c_{pq}^{\psi} \label{direc_invar},
\end{align}
with the last equality following from the fact that $\eval{d\diff}{p} \cdot \overline{\eval{d \diff}{p}} = 1$. Similarly, it follows that 
\begin{align}
z_{\diff(p)}^{\diff \psi} \left[z_{\diff(p)}^{\diff \psi}\right]^{*}   & \stackrel{(\ref{latent_vector_equiv})}{=}  \eval{d \diff}{p} \cdot \overline{\eval{d\diff}{p}} \cdot z_{p}^{\psi} \left[z_{p}^{\psi}\right]^{*} \nonumber \\
& \stackrel{\phantom{(\ref{latent_vector_equiv})}}{=} z_{p}^{\psi} \left[z_{p}^{\psi}\right]^{*}. \label{invar_invar}
\end{align}
Now, denoting $\nfield$ as the decoder neural field such that
\begin{align}
\widehat{\psi}_p (q) \equiv \nfield\left(c^{\psi}_{pq}, \textrm{vec}_{j \geq i}\left( z_{p}^{\psi} \left[z_{p}^{\psi}\right]^{*}\right) \right),\label{pred_def}
\end{align}
it follows that for all $q' \in \neigh{\diff(p)} \subset N$ we have 
\begin{align*}
\widehat{[\diff \psi]}_{\diff(p)} (q') & \stackrel{(\ref{pred_def})}{=}  \nfield\left(c^{\diff \psi}_{\diff(p)q'}, \textrm{vec}_{j \geq i}\left( z_{\diff(p)}^{\diff(\psi)} \left[z_{\diff(p)}^{\diff \psi}\right]^{*}\right) \right) \\
& \stackrel{\phantom{(\ref{pred_def})}}{=} \nfield\left(c^{\psi}_{p\diff^{-1}(q')}, \textrm{vec}_{j \geq i}\left( z_{p}^{\psi} \left[z_{p}^{\psi}\right]^{*}\right) \right) \\
& \stackrel{(\ref{pred_def})}{=} \diff \widehat{\psi}_p(q'), 
\end{align*}
with the equality $c^{\diff \psi}_{\diff(p)q'} = c^{\psi}_{p\diff^{-1}(q')}$ following from Equation~(\ref{direc_invar}) and the inveritibility of isometries.
Thus we have $\widehat{\left[\diff \psi\right]}_{\diff(p)} = \diff \widehat{\psi}_{p}$ as desired. 
\end{proof}

\subsection{Field Latent Diffusion Models}
\label{supp:fldm_equiv}
Consider the denoising network $\varepsilon$ as in Equation~(\ref{denoising_network}) and the forward and reverse diffusion processes defined in Equations~(\ref{diff_forward}) and (\ref{reverse_relation}). 

\begin{claim}
If for all latent vector fields $Z \in TM^d$, isometries $\diff: M \rightarrow N$, embeddings $\rho \in \nice{M}{\R^m}$, and timestep $ 0 \leq t \leq T$, the denoising network satisfies the condition of Equation~(\ref{diff_equiv_cond})
\begin{align*}
\diff \varepsilon(Z, t, \rho) = \varepsilon( \gamma Z, t, \gamma \rho),
\end{align*}
then \textit{both} the forward and reverse diffusion processes commute with isometries with
\begin{align*}
\diff Z_t = \left[\diff Z\right]_t \qquad \textrm{and} \qquad \diff \widetilde{Z}_t = \big[\diff \widetilde{Z}\big]_t.
\end{align*}
\end{claim}

\noindent \textbf{Note:} At the risk of abusing notation, we use the subscript $t$ to refer to both the forward and reverse diffusion processes. To disambiguate the two, we denote latent vector fields sampled from the reverse process with a tilde, \textit{e.g.} $\widetilde{Z}_t \in TM^d$. 

\begin{proof}
Consider any latent vector fields $Z \in TM^d$ and isometry $\diff:M \rightarrow N$. The forward process does not depend on the denoising network and demonstrating equivariance is straight-forward as 
\begin{align*}
\diff Z_t & \stackrel{(\ref{diff_forward})}{=} \sqrt{\alpha_{t}} \diff Z + \sqrt{1 - \alpha_t} \diff \epsilon, \quad \epsilon \sim \TND{M} \\ 
 & \stackrel{(\ref{equiv_noise})}{=} \sqrt{\alpha_{t}} \diff Z + \sqrt{1 - \alpha_t} \epsilon', \quad \epsilon' \sim \TND{N} \\ 
& \stackrel{(\ref{diff_forward})}{=} \left[\diff Z\right]_t
\end{align*}
Now, suppose the denoising network $\varepsilon$ satisfies the condition. Equivariance of the reverse process can be shown through induction, with the base case $t = T$ holding as
\begin{align*}
\diff \widetilde{Z}_T & \stackrel{\phantom{(\ref{equiv_noise})}}{=} \diff \widetilde{Z}, \quad \widetilde{Z} \sim \TND{M} \\
& \stackrel{(\ref{equiv_noise})}{=} \diff \widetilde{Z}, \quad \diff \widetilde{Z} \sim \TND{N} \\
& \stackrel{\phantom{(\ref{equiv_noise})}}{=} \big[\diff \widetilde{Z}]_T
\end{align*}
Then, assuming $\diff \widetilde{Z}_\tau = \big[\diff \widetilde{Z}\big]_\tau$ holds at step $\tau$, for step $\tau - 1$ we have 
\begin{align*}
\diff \widetilde{Z}_{\tau-1} & \stackrel{(\ref{reverse_relation})}{=} C_1(\alpha_\tau, \alpha_{\tau-1}) \, \diff \widetilde{Z}_\tau  \\
& \phantom{\stackrel{(\ref{reverse_relation})}{=}} + C_2(\alpha_\tau, \alpha_{\tau-1}) \, \diff \varepsilon(\widetilde{Z}_\tau, \tau, \rho)  \\
& \phantom{\stackrel{(\ref{reverse_relation})}{=}} + C_3(\alpha_\tau, \alpha_{\tau-1}) \, \diff \epsilon, \quad \epsilon \sim \TND{M} \\
& 
\labsym{=}{(\ref{equiv_noise})}{(\ref{reverse_relation})} C_1(\alpha_\tau, \alpha_{\tau-1}) \, \diff \widetilde{Z}_\tau \\
 & \phantom{\stackrel{(\ref{reverse_relation})}{=}} + C_2(\alpha_\tau, \alpha_{\tau-1}) \, \diff \varepsilon(\widetilde{Z}_\tau, \tau, \rho)  \\
& \phantom{\stackrel{(\ref{reverse_relation})}{=}} + C_3(\alpha_\tau, \alpha_{\tau-1}) \, \epsilon', \quad \epsilon' \sim \TND{N} \\
& \stackrel{(\ref{diff_equiv_cond})}{=} C_1(\alpha_\tau, \alpha_{\tau-1}) \, \diff \widetilde{Z}_\tau  \\
& \phantom{\stackrel{(\ref{reverse_relation})}{=}} + C_2(\alpha_\tau, \alpha_{\tau-1}) \, \varepsilon(\diff \widetilde{Z}_\tau, \tau, \diff \rho)  \\
& \phantom{\stackrel{(\ref{reverse_relation})}{=}} + C_3(\alpha_\tau, \alpha_{\tau-1}) \, \epsilon' \\
& \stackrel{\phantom{(\ref{diff_equiv_cond})}}{=} C_1(\alpha_\tau, \alpha_{\tau-1}) \, \big[\diff \widetilde{Z}\big]_\tau \\
& \phantom{\stackrel{(\ref{reverse_relation})}{=}} + C_2(\alpha_\tau, \alpha_{\tau-1}) \, \varepsilon(\big[\diff \widetilde{Z}\big]_\tau, \tau, \diff \rho)  \\
& \phantom{\stackrel{(\ref{reverse_relation})}{=}} + C_3(\alpha_\tau, \alpha_{\tau-1}) \, \epsilon' \\ 
& \stackrel{\phantom{(\ref{diff_equiv_cond})}}{=} \big[\diff \widetilde{Z}]_{\tau - 1},
\end{align*}
with the second to last equality following from the induction assumption. Thus, by induction we have $\diff \widetilde{Z}_t = \big[\diff \widetilde{Z}\big]_t$ for $0 \leq t \leq T$ as desired. 
\end{proof}

%% file: supplement/architectures.tex
\section{Implementation Details}

\label{impl_details}
\paragraph{Tangent bases} At each vertex, the associated tangent space lies in a plane perpendicular to the vertex normal direction.  We construct orthonormal bases in the tangent space at a given vertex by simply choosing an arbitrary unit vector perpendicular to the normal as the $x-$axis, then taking its cross product with the normal to get the $y-$axis.

\paragraph{Linearities, Nonlinearities, and Normalization}
With the exception of the decoder neural field $\nfield$ as in Equation~(\ref{pred_def}) and a few layers in the encoder, all of our networks are complex --- all linear layers and convolutional filters have complex-valued weights, with additive biases omitted to preserve equivariance. 

For both FL-VAE and FLDM, we use modified versions of  Vector Neurons (\textbf{VNs}) \cite{deng2021vector} and Filter Response Normalization \cite{singh2020filter} as nonlinear and normalization layers for tangent vector features. Nonlinearities are learnable mappings between $TM^C$ and $TM^{C'}$ applied point-wise. Letting $K, Q \in TM^{C'}$ be the output of two learnable linear layers taking features $X \in TM^C$ as input, the $c'-$th output feature $X'_{c'} \in TM$ of our nonlinear layers is given by 
\begin{align}
X'_{c'}(p) = Q_{c'}(p) + \frac{\textrm{ReLU}\left(-\textrm{Re}(\overline{Q_{c'}(p)}K_{c'}(p))\right)}{\abs{K_{c'}(p)}^2 + \delta} K_{c'}(p), \label{vn_nonlin}
\end{align}
where $\delta > 0$ a constant and the argument of the $\textrm{ReLU}$ function is the inner product of $Q_{c'}(p)$ and $K_{c'}(p)$ in $T_pM$. 

Similarly, given tangent vector features $X \in TM^C$, normalization is applied on a per-channel basis independent of the batch size via the mapping 
\begin{align}
X_c \mapsto \frac{a_c X_c}{\sqrt{\E_{p \in M} \abs{X_c(p)}^2 + \delta_c}}, \label{frn_norm}
\end{align}
where $a_c \in \C$ and $\delta_c \in \R_{> 0}$ are learnable parameters. 

It is easy to see that both our normalization and nonlinear layers are equivariant under isometries, as they preserve magnitudes, inner products, and areas (which preserves the computation of means).  
\subsection{Field Latents}
\label{fl_architecture}

\paragraph{Encoder architecture} The encoder operates pointwise, and is carefully constructed so as to satisfy the conditions for equivariance in Equation~(\ref{fl_equiv_cond}).
At each vertex $p$, a fixed number of points are uniformly sampled in the one ring neighborhood at which the texture is evaluated. The resulting collection of $3-$channel scalar features $\{\psi(q_i)\}$ are lifted to a collection of $C-$channel tangent vector features $\{ \zeta_{pq_i}^{\psi} \}$ by multiplying the logarithms by the output of a linear layer ${\mathcal{L}_1}: \R^3 \rightarrow \C^C$ such that 
\begin{align}
\zeta_{pq_i}^{\psi} = {\mathcal{L}}_1(\psi(q_i)) \cdot \log_p q_i
\end{align}
A token feature $\xi_p^{\psi} \in \C_p^C$ is initialized via a gathering operation with 
\begin{align}
\xi_p^{\psi} = \sum_{i} \log_p q_i  \cdot {\mathcal{L}_2}(\psi(q_i)) \odot f(\abs{\log_p q_i}) ,
\end{align}
where ${\mathcal{L}}_2: \R^3 \rightarrow \C^C$ is a linearity and $f: \R \rightarrow \C^C$ is a filter parameterized by a two-layer MLP. It is easy to see that the lifted features satisfy 
\begin{align}
\eval{d \diff}{p} \cdot \zeta_{pq_i}^{\psi} = \zeta_{\diff(p)\diff(q_i)}^{\diff \psi} \quad \textrm{and} \quad \eval{d \diff}{p} \cdot \xi_{p}^{\psi} = \xi_{\diff(p)}^{\diff \psi}
\end{align}
for any isometry $\diff: M \rightarrow N$ due to the transformation of the logarithm map and the preservation of its magnitude. 

 The concatenation $\{ \zeta_{pq_i}^{\psi} \} \cup \{ \xi_p^{\psi} \}$ is then passed to eight successive VN-Transformer layers \cite{assaad2022vn}; each layer consists of a normalization, followed by a multi-headed attention block over the sample dimension, a second normalization, and two tangent vector neurons, with residual connections after the attention layer and final nonlinearity. The linear layers in the attention block use complex weights, with the the real part of product between the keys and values passed to the softmax --- a  construction that is equivariant under both isometries and the ordering of samples.

Afterwards, the token feature is extracted from the output of the VN-Transformer layers $\{ \widehat{\zeta}_{pq_i}^{\psi} \} \cup \{ \widehat{\xi}_p^{\psi} \}$ and used to predict the mean and standard deviation. Specifically, the token feature is passed to a two-layer equivariant MLP $\mathcal{M}_1: \C^{C} \rightarrow \C^{d}$ using VN activations to predict the mean with \begin{align}
\mu_p^{\psi} & = \mathcal{M}_1\left(\widehat{\xi}_p^{\psi}\right). 
\end{align}
An invariant scalar feature is constructed from the token with a learnable product \cite{sharp22diffusionnet} and used to predict the log of the standard deviation with 
\begin{align} 
\textrm{ln} \, \sigma_p^{\psi} = \mathcal{M}_2\left(\overline{\widehat{\xi}_{p}^{\psi}} \odot \mathcal{L}_3(\widehat{\xi}_p^{\psi}) \right),
\end{align}
where $\mathcal{L}_3: \C^C \rightarrow \C^C$ is a complex linearity and $\mathcal{M}_2: \C^{C} \rightarrow \R^d$ is a two-layer real-valued MLP with SiLU activations, taking the stacking of the real and imaginary components of the invariant feature as input.

In practice, we sample $64$ uniformly distributed points in each one-ring, and use $C = 128$ channels in all layers of the encoder network. 

\paragraph{Decoder architecture}
The learnable component of the FL decoder is the neural field $\nfield$ to which the concatenation of the positional and invariant features are passed to predict the texture value as in Equation~(\ref{pred_def}). Specifically, the neural field is a five-layer real-valued MLP with SiLU activations and $512$ channels in each layer, taking the stacking of the real and imaginary components of the complex features as input.

\paragraph{Training}

During training we consider the reconstruction loss over the one-rings 
\begin{align}
L_{\textrm{R}} = \E_{ \, p \in M, q \in \neigh{p}} \, \norm{\psi(q) - \widehat{\psi}_p(q)}_1, \label{latent_recon_loss}
\end{align}
in addition to the Kullback-Leibler divergence to regularize the distributions of the latent codes in the tangent space
\begin{align}
\begin{aligned}
&L_{\textrm{KL}} = \\ &\frac{1}{2} \, \E_{p \in M} \sum_{i=1}^d \bigg[1 + \textrm{ln} \left[\sigma_p\right]_i^2 
 -\abs{\left[\mu_p\right]_i}^2 - \left[\sigma_p\right]_i^2\bigg]. 
\end{aligned} \label{latent_kl_loss}
\end{align}
The FL-VAE is trained to optimize the sum of the reconstruction and KL losses,
\begin{align}
L_{\textrm{FL}} = L_{\textrm{R}} + \lambda L_{\textrm{KL}}, \label{latent_loss}
\end{align}
with $\lambda > 0$ controlling the weight between the terms. 

We train a single FL-VAE for use in all experiments. It is trained with $\lambda = 0.01$. Training is performed for $1.5$M iterations with a batch size of $16$ and an initial learning rate of $10^{-3}$, decaying to $10^{-5}$ on a cosine schedule. Running on an NVIDIA A100 GPU, convergence usually happens within five hours. 

The FL-VAE is trained using OpenImagesV4 \cite{kuznetsova2020open} which contains \emph{only} images. Planar meshes are randomly generated using farthest point sampling followed by Delaunay triangulation, and are are overlaid on these images. The FL-VAE is trained on the one-rings of the resulting textured 2D meshes. By design, the FL-VAE is agnostic to the 3D embedding of one-rings because it only interfaces with their local flattening, allowing for zero-shot encoding and decoding on arbitrary 3D meshes at inference.

As a general rule, the number and variance of pixel values in one-rings during training should be equal to or greater than what is expected on target 3D assets to ensure robustness (we train on a high-res, diverse image dataset OpenImagesV4 to ensure high-variance).  Furthermore the latent dimension should increase proportional to the number of pixels expected in each one-ring, which is dependent on both mesh and texture resolution.

As the FL-VAE does not see highly asymmetric one-rings at training, to improve robustness at inference we uniformly remesh target 3D assets (though varying the planar mesh quality during training may achieve a similar result).

\subsection{Field Latent Diffusion Models}
\label{fldm_architecture}

\paragraph{Reverse process} The values of the coefficients in the reverse process in Equation~(\ref{reverse_relation}) are given by \cite{ho2020denoising, song2020denoising}:
\begin{align}
C_1(\alpha_t, \alpha_{t-1}) & = \sqrt{\frac{\alpha_{t-1}}{\alpha_t}} \\
C_2(\alpha_t, \alpha_{t-1}) & =   \sqrt{\frac{\alpha_{t-1}(1 - \alpha_t)}{\alpha_t}} \sqrt{\frac{(1 - \alpha_{t-1})^2 \alpha_t}{\alpha_{t-1}(1 - \alpha_t)}} \\ 
C_3(\alpha_t, \alpha_{t-1}) &= \sqrt{\frac{1 - \alpha_{t-1}}{1 - \alpha_t} \left(1 - \frac{\alpha_t}{\alpha_{t-1}}\right)}
\end{align}

\paragraph{Choice of convolution operator}

A majority of existing, state-of-the-art surface convolution operators are designed process scalar features, either alone \cite{sharp22diffusionnet, hu2022subdivision} or in tandem with tangent vector features in a multi-stream approach \cite{Wiersma2020, wiersma2022deltaconv}.  To our knowledge, field convolutions \cite{mitchel2021field} are the only existing state-of-the-art surface convolution designed specifically to process tangent vector features. The richness of the convolution operators in the multi-stream approaches is due to the intermixing of features, and the tangent vector convolution operators themselves are individually undiscriminating. More generally, field convolutions are part of a larger framework for equivariant convolutions that has demonstrated success across a wide variety of modalities and applications \cite{mitchel2021echo, mitchelExtConv2020, mitchel2022mobius, mitchel2022ext}.  

However, DiffusionNet \cite{sharp22diffusionnet} is relatively fast for a surface network and connectivity agnostic --- attractive properties for a denoising network. We experimented with extending it  to handle tangent vector features by replacing the Laplacian eigenfunctions with those of the vector Lapalacian and adapting the network with the nonlinearities and normalization in Equations~(\ref{vn_nonlin}) and (\ref{frn_norm}). While we were able to successfully train a denoising model with this architecture, we found that decoded samples had very little diversity and exhibited a ``raggedness'' similar to the textures reconstructed from INFs \cite{koestler2022intrinsic} as seen in Figure~\ref{tex_recon_fig}. The lack of diversity in the generated samples is likely due to the fact that global support is baked-in to DiffusionNet's ``convolutions'', making it challenging to prevent it from overfitting to a single example.

Furthermore, DiffusionNet bandlimits features by projecting them to a low-dimensional (vector) Laplacian eigenbasis. This presents several problems when processing vector fields in FL space. Specifically, the vector Laplacian eigenbasis spans \textit{continuous} vector fields. The FL encoder constructs latent vector fields pointwise, which are not guaranteed to be continuous, let alone smooth. Thus, vector fields in FL space are likely to suffer significant deterioration by projecting onto the low-frequency eigenbasis, potentially leading the material degradation of quality observed in reconstructed samples.

\paragraph{Extending field convolutions}
Field convolutions convolve vector fields $X \in TM^C$ with $C' \times C$ filter banks $\vn{f}_{c'c} \in \nice{\C}{\C}$, returning vector fields $X * \vn{f} \in TM^{C'}$. To simplify notation, we convert to polar coordinates: expressing the $c-$th input vector field and the logarithm in the tangent space at  a point $p \in M$ as
\begin{align}
X_c(p) \equiv \varrho^{c}_p \, e^{i \phi^{c}_p} \qquad \textrm{and} \qquad \log_p q \equiv r_{pq} e^{i\theta_{pq}}.
\end{align}
Similarly, we denote by $\varphi_{pq}$ the angle of rotation resulting from the parallel transport $\PT{p}{q}: T_qM \rightarrow T_pM$ along the shortest geodesic from neighboring points $q$ to $p$ on the surface, such that for any $\vn{v} \in T_qM$, 
\begin{align}
\PT{p}{q}(\vn{v}) \equiv e^{i\varphi_{pq}} \vn{v}
\end{align}
Then, the $c'-$th output of the field convolution of $X$ with $\textbf{f}$ is the vector field \cite{mitchel2021field}
\begin{align}
\begin{aligned}
&  [\, X * \vn{f} \, ]_{c'} (p) = \\ 
& \sum_{c = 0}^{C} \int_{\neigh{p}} \varrho_q^c \, e^{i(\phi_q^{c} + \varphi_{pq})} \, \vn{f}_{c'c}\left( r_{pq} \, e^{i(\theta_{pq} - \phi_q^{c})} \right) dq 
\end{aligned}
\label{field_conv}
\end{align}

We extend field convolutions to interleave scalar embeddings with tangent vector features by convolving vector fields with filters $\vn{f}_{c'c} \in \nice{\C \times \R^e}{\C}$, with $e$ the embedding dimension. Specifically, given a scalar embedding $\pi \in \nice{M}{\R^{e}}$, the $c'-$th output of field convolution of vector fields $X \in TM^C$ and embedding $\pi$ with $C' \times C$ filter banks $\vn{f}_{c'c}  \in \nice{\C \times \R^e}{\C}$ is the vector field 
\begin{align}
\begin{aligned}
&[\, \{X, \pi \} * \vn{f} \, ]_{c'} (p) =  \\
& \sum_{c = 0}^{C} \int_{\neigh{p}} \varrho_q^c \, e^{i(\phi_q^{c} + \varphi_{pq})} \, \vn{f}_{c'c}\left( r_{pq} \, e^{i(\theta_{pq} - \phi_q^{c})}, \pi(q) \right) dq \label{field_conv_emb}
\end{aligned}
\end{align}
For any isometry $\diff: M \rightarrow N$, field convolutions have the property \cite{mitchel2021field}
\begin{align}
\diff X * \vn{f} = \diff [ X * \vn{f}],
\end{align}
due in part to the invariance of the filter arguments under isometries. Thus, it can be shown that the extension of field convolutions in Equation~(\ref{field_conv_emb}) is equivariant in the sense that 
\begin{align}
\{\diff X, \diff \pi\} * \vn{f} = \diff [\, \{X, \pi \} * \vn{f} \,].
\end{align}

Here, our implementation of field convolutions differs slightly from that of the original. The filter support is restricted to the immediately adjacent vertices in the one-ring about a point. Additionally, instead of parameterizing filters with Fourier basis functions, we make use of the PointConv trick \cite{wu2019pointconv, finzi2020generalizing} to efficiently represent filters $\vn{f}_{c'c}$ as three-layer MLPs with eight channels in the interior layers. 

\paragraph{Denoising network}
The denoising network is constructed with FCResNet blocks \cite{mitchel2021field}, modified to inject scalar embeddings $\pi \in \nice{M}{\R^{e}}$ derived from the diffusion timestep and optional user-input features. Each FCResNet block consists of a three-layer VN MLP, with normalization before each VN, followed by a standard field convolution, a second three-layer VN MLP, and a final field convolution with the embedding. 

The architecture of the denoising model is based on the observation that by limiting the size of the network's receptive field, distributions of latent features can be learned without overfitting to the single textural example \cite{wu2023sin3dm, kulikov2023sinddm, wang2022sindiffusion}.  Following \cite{wang2022sindiffusion} the denoising network $\varepsilon$ takes the form of a shallow two-level U-Net with a single downsampling layer. The input layer consists of two FCResNet Blocks, with four FCResNet blocks in the downsampling and final upsampling layers each, for a total of $10$. The receptive field of the network is equivalent to a $28-$ring about each point, which is approximately $6\%$ of the total area of a $30$K vertex mesh. Upsampling and downsampling are performed using mean pooling and nearest-neighbor unpooling, respectively, augmented with parallel transport to correctly express tangent vector features in neighboring tangent spaces. In all experiments, our denoising networks use $96$ channels per FCResNet block at the finest resolution, increasing to $192$ after downsampling. 

\paragraph{Scaling input latents} In practice, we scale latent vector fields in the diffusion process to have magnitudes proportional to sampled noise in the tangent bundle.  From the definition of Gaussian noise in the tangent bundle in Equation~(\ref{tangent_noise}) as pointwise samples from 2D Gaussians, the expected value of the pointwise magnitude of  $X \sim \mathcal{TN}_M(0, I_1)$ follows the chi distribution with 
\begin{align}
\sqrt{\frac{\pi}{2}} = \E_{p \in M} \abs{X(p)}.
\end{align}
Thus, we scale input latents $Z \in TM^d$ to the diffusion model per-channel such that 
\begin{align}
Z_{i} \mapsto \frac{\sqrt{\pi}}{\sqrt{2}  \, \E_{p \in M} \abs{Z_{i}(p)}} \cdot Z_{i} 
\end{align}
for $1 \leq i \leq d$ to ensure they are represented at the same scale as the added noise. 

\paragraph{Training}

The denoising network is trained subject to the loss in Equation~(\ref{FLDM_loss})  for $300$K iterations, using the Adam \cite{kingma2017adam} optimizer with a batch size of $16$ and initial learning rate of $10^{-3}$, decaying to $10^{-5}$ on a cosine schedule. A computation bottleneck is the memory overhead induced by the irregular (one-ring) kernel support for FCs. To scale using a NVIDIA-V100 with only 16GB of memory, we must rematerialize tensors during backprop, resulting in slower training ($\sim$2 days w/8 V100s). This cost would be mitigated with a higher memory GPU. In contrast, Sin3DM \cite{wu2023sin3dm} reports training in several hours on an A6000 GPU.  For inference, sampling new textures is much more efficient and takes just 10 minutes (compared to 3-5 for Sin3DM).  While we found other popular surface networks to be unsuitable for FLDMs, FCs are not a fundamental part of our framework, and in the future could be replaced by a more efficient operator.

More generally, we found that training FLDMs on high-res meshes ($>$10K  vertices) produced more diverse and higher quality textures than training on low-res meshes. This is possibly because the FLDM sees a larger number of latent codes and each code characterizes a smaller ``unit'' of texture, affording greater flexibility in generating patterns and smoothing transitions.

\paragraph{Directional control} 
We expect that user-specified vector fields could also be passed to FLDMs to provide directional control, with embeddings formed by taking dot products with the network features. For example, the user may choose to condition the FLDMs in Fig. \ref{transfer_fig} with a vector field pointing about the whorl of the shell. The same field could be used during sampling (or, alternatively, a rotated or otherwise manipulated vector field) to control the orientation of synthesized features in the desired way.

%% file: supplement/more_results.tex
\section{Additional Results}
See the last pages of this supplementary material document for additional visualizations and comparisons.  

\newpage

\begin{figure*}[!ht]
\begin{picture}(\linewidth,0.66\columnwidth)
\put(0,0){{\includegraphics[width=\linewidth]{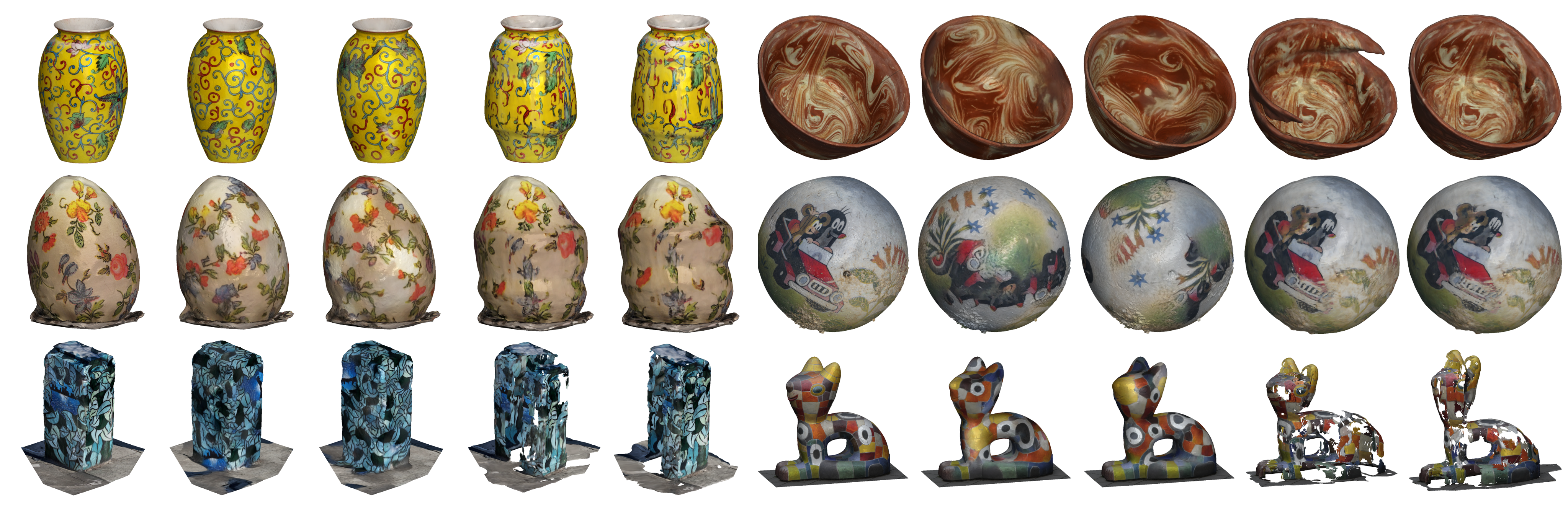}}}
\thicklines
\put(15, 165){ {\small Input}}
\put(8, 162){\line(1,0){38}} 
\put(82, 165){ {\small \textbf{FLDM}}}
\put(60, 161){\line(1,0){78}} 
\put(169, 165){{\small Sin3DM \cite{wu2023sin3dm}}}
\put(153, 161){\line(1,0){78}} 
\put(252, 165){ {\small Input}}
\put(242, 161){\line(1,0){44}} 
\put(322, 165){ {\small \textbf{FLDM}}}
\put(294, 161){\line(1,0){94}} 
\put(419, 165){{\small Sin3DM \cite{wu2023sin3dm}}}
\put(397, 161){\line(1,0){94}} 
\end{picture}
\caption{ Additional samples generated from FLDMs and Sin3DM \cite{wu2023sin3dm} in the unconditional paradigm described in section~\ref{sec:unc_gen}. The electrical box and cat sculpture on the bottom row are not included in our evaluations, but present significant failure cases for Sin3DM, which can struggle to synthesize samples from thin, shell-like models of the kind often arising from real-world 3D scans. \textit{Zoom in to compare.} }
\label{unc_gen_supp_fig}
\end{figure*}
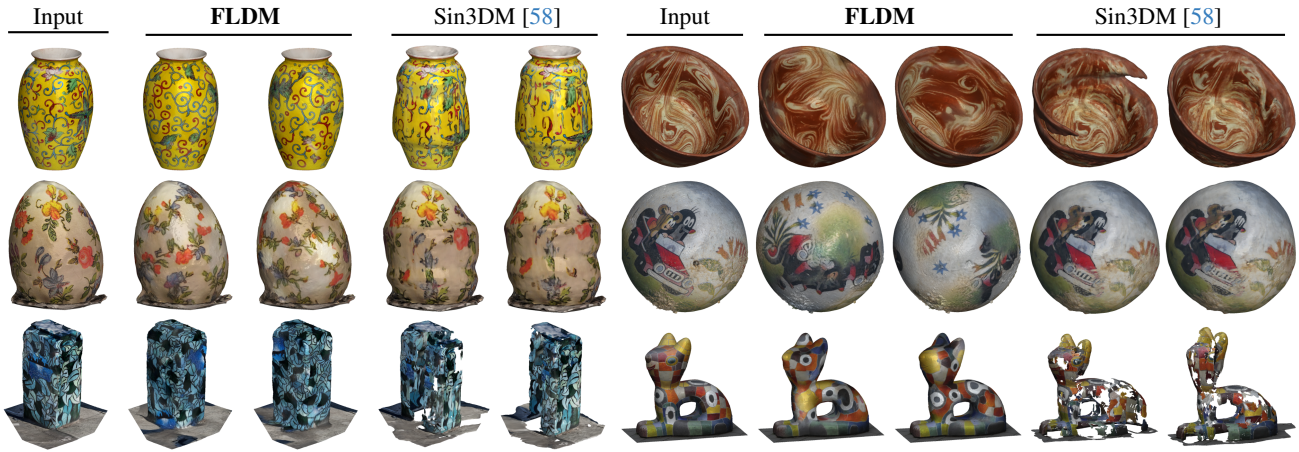

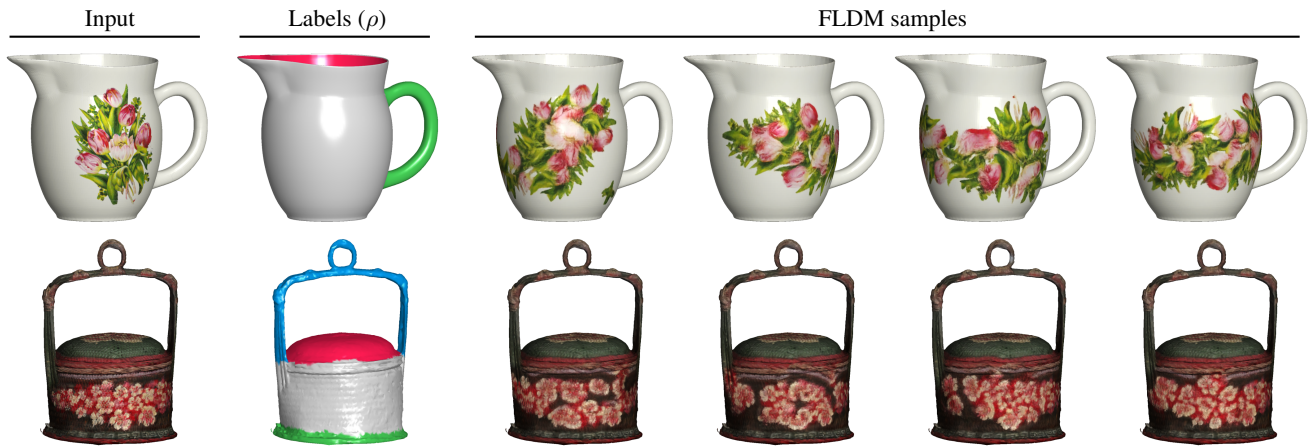
\begin{figure*}[!ht]
\begin{picture}(\linewidth,0.69\columnwidth)
\put(0,0){{\includegraphics[width=\linewidth]{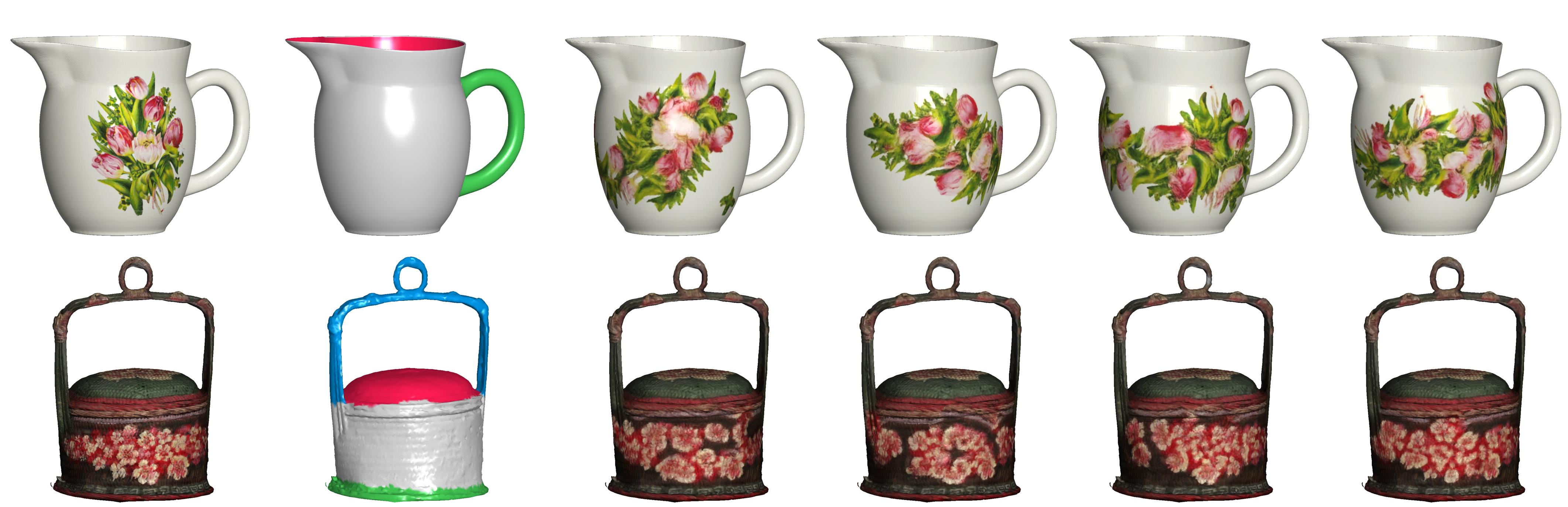}}}
\thicklines
\put(30, 160){ {\small Input}}
\put(4, 156){\line(1,0){72}} 
\put(107, 160){ {\small Labels ($\rho$)}}
\put(91, 156){\line(1,0){72}} 
\put(310, 160){{\small FLDM samples}}
\put(180, 156){\line(1,0){312}} 
\end{picture}
\caption{ Additional examples of label-guided texture synthesis with FLDMs. Many meshes contain interesting textural details only in specific regions and are otherwise uniformly colored. Label guidance can be used to ensure generated textures generally reflect the distribution of content on the training mesh. \textit{Zoom in to view.} }
\label{label_guided_supp_fig}
\end{figure*}

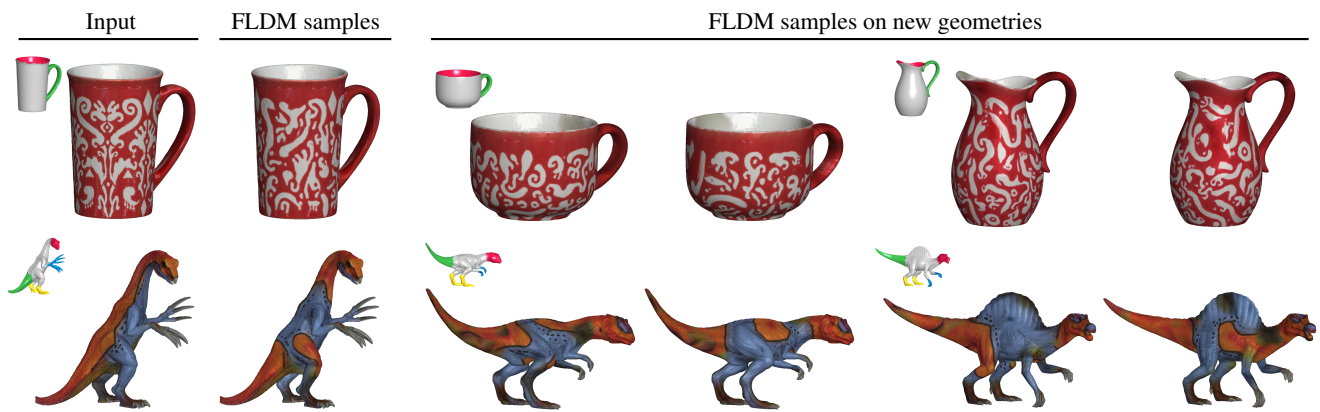
\begin{figure*}[!ht]
\begin{picture}(\linewidth,0.67\columnwidth)
\put(0,0){{\includegraphics[width=\linewidth]{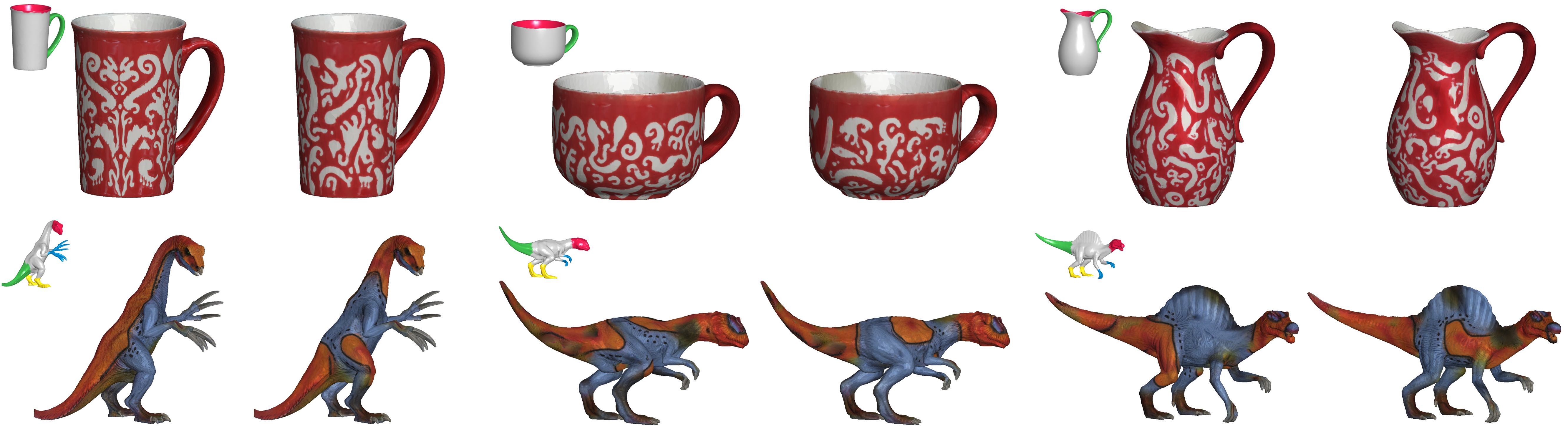}}}
\thicklines
\put(27, 147){ {\small Input}}
\put(4, 143){\line(1,0){68}} 
\put(82, 147){ {\small FLDM samples}}
\put(80, 143){\line(1,0){65}} 
\put(265, 147){{\small FLDM samples on new geometries}}
\put(160, 143){\line(1,0){332}} 
\end{picture}
\caption{ Additional examples of generative texture transfer with FLDMs. The FL-VAE and FLDMs are in fact equivariant under \textit{local} isometries, and are thus able to replicate textural details on new geometries that are only \textit{locally} similar to the training mesh. \textit{Zoom in to view.} }
\label{transfer_supp_fig}
\end{figure*}

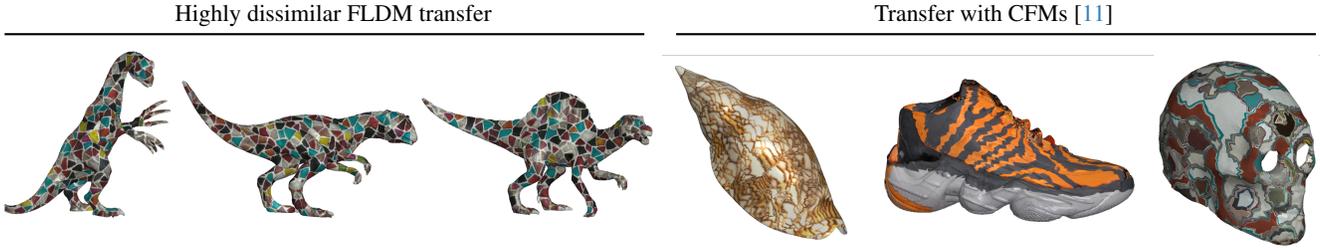
\begin{figure*}[!ht]
\begin{picture}(0.5\linewidth,0.25\columnwidth)
\put(0,10){{\includegraphics[width=0.5\linewidth]{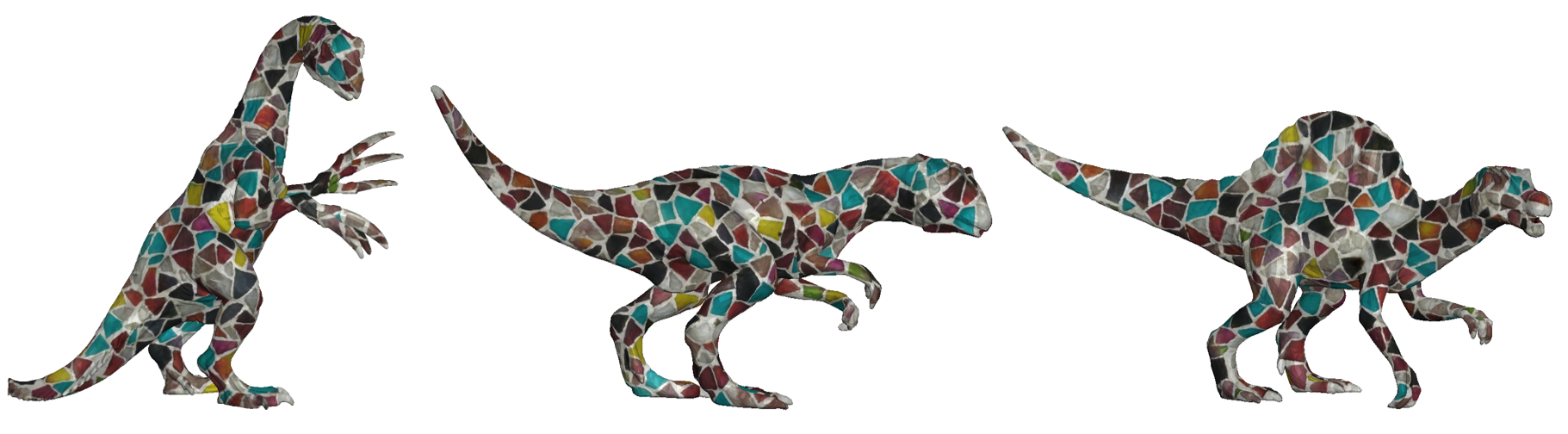}}}
\put(250,0){{\includegraphics[width=0.5\linewidth]{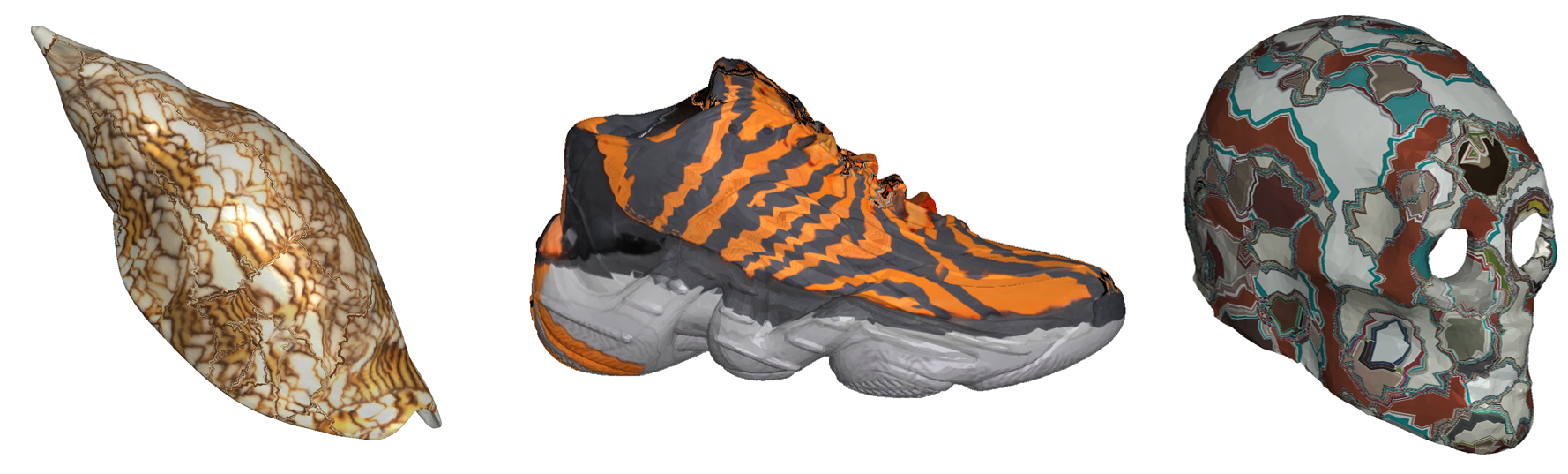}}}
\thicklines
\put(63, 88){ {\small Highly dissimilar FLDM transfer}}
\put(1, 83){\line(1,0){242}} 
\put(330, 88){{\small Transfer with CFMs \cite{donati2022complex}}}
\put(255, 83){\line(1,0){242}} 
\end{picture}
\caption{ \textit{Left:} While we believe the most practical use of generative texture transfer is intra-category, FL-VAEs and FLDMs can in fact transfer textures between  between highly dissimilar shapes such as from the \emph{skull} to \emph{dinosaur} meshes. \textit{Right:} Complex Functional Maps (\textbf{CFMs}) \cite{donati2022complex} transfer textures based on estimated dense correspondences. Despite SoTA performance in correspondence tasks, CFMs struggle in the presence of topological changes, as in the skull, and pointwise mappings don't explicitly promote smooth transfers.}
\label{rebuttal_fig}
\end{figure*}


\begin{table*}{
    \centering
    \begin{tabular}{lc}
    ID & Visualization  \\ \hline \hline
    TROCHILUS\_BOOST & --- \\
    Schleich\_Therizinosaurus\_ln9cruulPqc & --- \\ 
    SAPPHIRE\_R7\_260X\_OC & --- \\   
    Tieks\_Ballet\_Flats\_Electric\_Snake & --- \\ 
    Reebok\_FS\_HI\_INT\_R12 & Figure~\ref{tex_recon_fig} (Top) \\
    Snack\_Catcher\_Snack\_Dispenser & --- \\ 
    Polar\_Herring\_Fillets\_Smoked\_Peppered\_705\_oz\_total & --- \\ 
    Olive\_Kids\_Mermaids\_Pack\_n\_Snack\_Backpack & --- \\ 
    Olive\_Kids\_Trains\_Planes\_Trucks\_Bogo\_Backpack & --- \\ 
    Olive\_Kids\_Dinosaur\_Land\_Munch\_n\_Lunch & --- \\
    Fruity\_Friends & --- \\ 
    Horse\_Dreams\_Pencil\_Case & Figure~\ref{tex_recon_fig} (Bottom) \\ 
    Horses\_in\_Pink\_Pencil\_Case & --- \\ 
    LEGO\_City\_Advent\_Calendar & --- \\ 
    Digital\_Camo\_Double\_Decker\_Lunch\_Bag & --- \\
    ASICS\_GELDirt\_Dog\_4\_SunFlameBlack & ---
    \end{tabular}
    \caption{List of Google Scanned Objects \cite{downs2022google} assets used in the texture compression and reconstruction evaluations in section~\ref{sec:eval_recon}.}
    \label{tab:recon_assets_list}
    } 
\end{table*}

\begin{table*}{
    \centering
    \begin{tabular}{lcc}
    ID & Source & Visualization  \\ \hline \hline
    3969fe35c6444293af27b976dca085a4  & Objaverse & Figure~\ref{unc_gen_supp_fig} (Top left) \\ %
    656f83a4d9224c29abc82ade2207d2ce & Objaverse & --- \\ 
    e7caba92073d4adba3477c21aa25e91f & Objaverse & Figure~\ref{unc_gen_fig} (Top left) \\  
    ed0afde32a194337ba1a18b1018f2d2e & Objaverse & Figure~\ref{unc_gen_fig} (Top right) \\     
    e5712bffe9714a6aaa148b6b262b748d & Objaverse & Figure~\ref{unc_gen_supp_fig} (Center right) \\ 
    f0e3c872f1984cf7a467645d9e0d3abd & Objaverse & Figure~\ref{unc_gen_fig} (Bottom left) \\ 
    2df58d08f5604c058d43e00677d325b6 & Objaverse & Figure~\ref{unc_gen_supp_fig} (Top right) \\ 
    4a6d4fa8eb83401ca9ac0446bad83c0a & Objaverse & Figure~\ref{unc_gen_fig} (Bottom right) \\
    190211a19360444699dccec9eda105e6 & Objaverse & Figure~\ref{unc_gen_supp_fig} (Center left) \\ 
    Now\_Designs\_Bowl\_Akita\_Black & Google Scanned Objects & --- 
    \end{tabular}
    \caption{List of Objaverse \cite{deitke2022objaverse} and Scanned Objects \cite{downs2022google} assets used in the unconditional generation evaluations in section~\ref{sec:unc_gen}.}
    \label{tab:unc_gen_list}
    } 
\end{table*}

%% file: main.bbl
\begin{thebibliography}{64}
\providecommand{\natexlab}[1]{#1}
\providecommand{\url}[1]{\texttt{#1}}
\expandafter\ifx\csname urlstyle\endcsname\relax
  \providecommand{\doi}[1]{doi: #1}\else
  \providecommand{\doi}{doi: \begingroup \urlstyle{rm}\Url}\fi

\bibitem[Assaad et~al.(2022)Assaad, Downey, Al-Rfou, Nayakanti, and Sapp]{assaad2022vn}
Serge Assaad, Carlton Downey, Rami Al-Rfou, Nigamaa Nayakanti, and Ben Sapp.
\newblock Vn-transformer: Rotation-equivariant attention for vector neurons.
\newblock \emph{arXiv preprint arXiv:2206.04176}, 2022.

\bibitem[Bokhovkin et~al.(2023)Bokhovkin, Tulsiani, and Dai]{bokhovkin2023mesh2tex}
Alexey Bokhovkin, Shubham Tulsiani, and Angela Dai.
\newblock Mesh2tex: Generating mesh textures from image queries.
\newblock \emph{arXiv preprint arXiv:2304.05868}, 2023.

\bibitem[Cao et~al.(2023)Cao, Kreis, Fidler, Sharp, and Yin]{cao2023texfusion}
Tianshi Cao, Karsten Kreis, Sanja Fidler, Nicholas Sharp, and Kangxue Yin.
\newblock Texfusion: Synthesizing 3d textures with text-guided image diffusion models.
\newblock In \emph{ICCV}, pages 4169--4181, 2023.

\bibitem[Chen et~al.(2023{\natexlab{a}})Chen, Siddiqui, Lee, Tulyakov, and Nie{\ss}ner]{chen2023text2tex}
Dave~Zhenyu Chen, Yawar Siddiqui, Hsin-Ying Lee, Sergey Tulyakov, and Matthias Nie{\ss}ner.
\newblock Text2tex: Text-driven texture synthesis via diffusion models.
\newblock \emph{arXiv preprint arXiv:2303.11396}, 2023{\natexlab{a}}.

\bibitem[Chen et~al.(2023{\natexlab{b}})Chen, Chen, Zhou, and Zhang]{chen2023shaddr}
Qimin Chen, Zhiqin Chen, Hang Zhou, and Hao Zhang.
\newblock Shaddr: Real-time example-based geometry and texture generation via 3d shape detailization and differentiable rendering.
\newblock \emph{arXiv preprint arXiv:2306.04889}, 2023{\natexlab{b}}.

\bibitem[Chen et~al.(2022)Chen, Yin, and Fidler]{chen2022auv}
Zhiqin Chen, Kangxue Yin, and Sanja Fidler.
\newblock Auv-net: Learning aligned uv maps for texture transfer and synthesis.
\newblock In \emph{CVPR}, pages 1465--1474, 2022.

\bibitem[Cheng et~al.(2023)Cheng, Li, Liu, and Wang]{cheng2023tuvf}
An-Chieh Cheng, Xueting Li, Sifei Liu, and Xiaolong Wang.
\newblock Tuvf: Learning generalizable texture uv radiance fields.
\newblock \emph{arXiv preprint arXiv:2305.03040}, 2023.

\bibitem[Deitke et~al.(2022)Deitke, Schwenk, Salvador, Weihs, Michel, VanderBilt, Schmidt, Ehsani, Kembhavi, and Farhadi]{deitke2022objaverse}
Matt Deitke, Dustin Schwenk, Jordi Salvador, Luca Weihs, Oscar Michel, Eli VanderBilt, Ludwig Schmidt, Kiana Ehsani, Aniruddha Kembhavi, and Ali Farhadi.
\newblock Objaverse: A universe of annotated 3d objects, 2022.

\bibitem[Deitke et~al.(2023)Deitke, Liu, Wallingford, Ngo, Michel, Kusupati, Fan, Laforte, Voleti, Gadre, VanderBilt, Kembhavi, Vondrick, Gkioxari, Ehsani, Schmidt, and Farhadi]{deitke2023objaversexl}
Matt Deitke, Ruoshi Liu, Matthew Wallingford, Huong Ngo, Oscar Michel, Aditya Kusupati, Alan Fan, Christian Laforte, Vikram Voleti, Samir~Yitzhak Gadre, Eli VanderBilt, Aniruddha Kembhavi, Carl Vondrick, Georgia Gkioxari, Kiana Ehsani, Ludwig Schmidt, and Ali Farhadi.
\newblock Objaverse-xl: A universe of 10m+ 3d objects, 2023.

\bibitem[Deng et~al.(2021)Deng, Litany, Duan, Poulenard, Tagliasacchi, and Guibas]{deng2021vector}
Congyue Deng, Or Litany, Yueqi Duan, Adrien Poulenard, Andrea Tagliasacchi, and Leonidas~J Guibas.
\newblock Vector neurons: A general framework for so (3)-equivariant networks.
\newblock In \emph{ICCV}, pages 12200--12209, 2021.

\bibitem[Donati et~al.(2022)Donati, Corman, Melzi, and Ovsjanikov]{donati2022complex}
Nicolas Donati, Etienne Corman, Simone Melzi, and Maks Ovsjanikov.
\newblock Complex functional maps: A conformal link between tangent bundles.
\newblock In \emph{Computer Graphics Forum}, pages 317--334. Wiley Online Library, 2022.

\bibitem[Dosovitskiy et~al.(2020)Dosovitskiy, Beyer, Kolesnikov, Weissenborn, Zhai, Unterthiner, Dehghani, Minderer, Heigold, Gelly, et~al.]{dosovitskiy2020image}
Alexey Dosovitskiy, Lucas Beyer, Alexander Kolesnikov, Dirk Weissenborn, Xiaohua Zhai, Thomas Unterthiner, Mostafa Dehghani, Matthias Minderer, Georg Heigold, Sylvain Gelly, et~al.
\newblock An image is worth 16x16 words: Transformers for image recognition at scale.
\newblock \emph{arXiv preprint arXiv:2010.11929}, 2020.

\bibitem[Downs et~al.(2022)Downs, Francis, Koenig, Kinman, Hickman, Reymann, McHugh, and Vanhoucke]{downs2022google}
Laura Downs, Anthony Francis, Nate Koenig, Brandon Kinman, Ryan Hickman, Krista Reymann, Thomas~B. McHugh, and Vincent Vanhoucke.
\newblock Google scanned objects: A high-quality dataset of 3d scanned household items, 2022.

\bibitem[Elhag et~al.(2023)Elhag, Susskind, and Bautista]{elhag2023manifold}
Ahmed~A Elhag, Joshua~M Susskind, and Miguel~Angel Bautista.
\newblock Manifold diffusion fields.
\newblock \emph{arXiv preprint arXiv:2305.15586}, 2023.

\bibitem[Fayer et~al.(2018)Fayer, Durix, Gasparini, and Morin]{fayer2018texturing}
Julien Fayer, Bastien Durix, Simone Gasparini, and G{\'e}raldine Morin.
\newblock Texturing and inpainting a complete tubular 3d object reconstructed from partial views.
\newblock \emph{Computers \& Graphics}, 74:\penalty0 126--136, 2018.

\bibitem[Finzi et~al.(2020)Finzi, Stanton, Izmailov, and Wilson]{finzi2020generalizing}
Marc Finzi, Samuel Stanton, Pavel Izmailov, and Andrew~Gordon Wilson.
\newblock Generalizing convolutional neural networks for equivariance to lie groups on arbitrary continuous data.
\newblock In \emph{ICML}, pages 3165--3176. PMLR, 2020.

\bibitem[Gao et~al.(2022)Gao, Shen, Wang, Chen, Yin, Li, Litany, Gojcic, and Fidler]{gao2022get3d}
Jun Gao, Tianchang Shen, Zian Wang, Wenzheng Chen, Kangxue Yin, Daiqing Li, Or Litany, Zan Gojcic, and Sanja Fidler.
\newblock Get3d: A generative model of high quality 3d textured shapes learned from images.
\newblock \emph{NeurIPS}, 35:\penalty0 31841--31854, 2022.

\bibitem[Gardner et~al.(2022)Gardner, Egger, and Smith]{gardner2022rotation}
James Gardner, Bernhard Egger, and William Smith.
\newblock Rotation-equivariant conditional spherical neural fields for learning a natural illumination prior.
\newblock \emph{NeurIPS}, 35:\penalty0 26309--26323, 2022.

\bibitem[Granot et~al.(2022)Granot, Feinstein, Shocher, Bagon, and Irani]{granot2022drop}
Niv Granot, Ben Feinstein, Assaf Shocher, Shai Bagon, and Michal Irani.
\newblock Drop the gan: In defense of patches nearest neighbors as single image generative models.
\newblock In \emph{CVPR}, pages 13460--13469, 2022.

\bibitem[Grigorev et~al.(2019)Grigorev, Sevastopolsky, Vakhitov, and Lempitsky]{grigorev2019coordinate}
Artur Grigorev, Artem Sevastopolsky, Alexander Vakhitov, and Victor Lempitsky.
\newblock Coordinate-based texture inpainting for pose-guided human image generation.
\newblock In \emph{CVPR}, pages 12135--12144, 2019.

\bibitem[Hinz et~al.(2021)Hinz, Fisher, Wang, and Wermter]{hinz2021improved}
Tobias Hinz, Matthew Fisher, Oliver Wang, and Stefan Wermter.
\newblock Improved techniques for training single-image gans.
\newblock In \emph{Proceedings of the IEEE/CVF Winter Conference on Applications of Computer Vision}, pages 1300--1309, 2021.

\bibitem[Ho et~al.(2020)Ho, Jain, and Abbeel]{ho2020denoising}
Jonathan Ho, Ajay Jain, and Pieter Abbeel.
\newblock Denoising diffusion probabilistic models.
\newblock \emph{NeurIPS}, 33:\penalty0 6840--6851, 2020.

\bibitem[Hu et~al.(2022)Hu, Liu, Guo, Cai, Huang, Mu, and Martin]{hu2022subdivision}
Shi-Min Hu, Zheng-Ning Liu, Meng-Hao Guo, Jun-Xiong Cai, Jiahui Huang, Tai-Jiang Mu, and Ralph~R Martin.
\newblock Subdivision-based mesh convolution networks.
\newblock \emph{ACM Transactions on Graphics (TOG)}, 41\penalty0 (3):\penalty0 1--16, 2022.

\bibitem[Kingma and Ba(2017)]{kingma2017adam}
Diederik~P. Kingma and Jimmy Ba.
\newblock Adam: A method for stochastic optimization, 2017.

\bibitem[Kingma and Welling(2013)]{kingma2013auto}
Diederik~P Kingma and Max Welling.
\newblock Auto-encoding variational bayes.
\newblock \emph{arXiv preprint arXiv:1312.6114}, 2013.

\bibitem[Kn{\"o}ppel et~al.(2013)Kn{\"o}ppel, Crane, Pinkall, and Schr{\"o}der]{knoppel2013globally}
Felix Kn{\"o}ppel, Keenan Crane, Ulrich Pinkall, and Peter Schr{\"o}der.
\newblock Globally optimal direction fields.
\newblock \emph{ACM TOG}, 32\penalty0 (4):\penalty0 1--10, 2013.

\bibitem[Koestler et~al.(2022)Koestler, Grittner, Moeller, Cremers, and L{\"a}hner]{koestler2022intrinsic}
Lukas Koestler, Daniel Grittner, Michael Moeller, Daniel Cremers, and Zorah L{\"a}hner.
\newblock Intrinsic neural fields: Learning functions on manifolds.
\newblock In \emph{ECCV}, pages 622--639. Springer, 2022.

\bibitem[Kulikov et~al.(2023)Kulikov, Yadin, Kleiner, and Michaeli]{kulikov2023sinddm}
Vladimir Kulikov, Shahar Yadin, Matan Kleiner, and Tomer Michaeli.
\newblock Sinddm: A single image denoising diffusion model.
\newblock In \emph{ICML}, pages 17920--17930. PMLR, 2023.

\bibitem[Kuznetsova et~al.(2020)Kuznetsova, Rom, Alldrin, Uijlings, Krasin, Pont-Tuset, Kamali, Popov, Malloci, Kolesnikov, et~al.]{kuznetsova2020open}
Alina Kuznetsova, Hassan Rom, Neil Alldrin, Jasper Uijlings, Ivan Krasin, Jordi Pont-Tuset, Shahab Kamali, Stefan Popov, Matteo Malloci, Alexander Kolesnikov, et~al.
\newblock The open images dataset v4: Unified image classification, object detection, and visual relationship detection at scale.
\newblock \emph{IJCV}, 128\penalty0 (7):\penalty0 1956--1981, 2020.

\bibitem[Liu et~al.(2023)Liu, Feng, Black, Nowrouzezahrai, Paull, and Liu]{liu2023meshdiffusion}
Zhen Liu, Yao Feng, Michael~J Black, Derek Nowrouzezahrai, Liam Paull, and Weiyang Liu.
\newblock Meshdiffusion: Score-based generative 3d mesh modeling.
\newblock \emph{arXiv preprint arXiv:2303.08133}, 2023.

\bibitem[Maggiordomo et~al.(2023)Maggiordomo, Cignoni, and Tarini]{maggiordomo2023texture}
A Maggiordomo, P Cignoni, and M Tarini.
\newblock Texture inpainting for photogrammetric models.
\newblock In \emph{Computer Graphics Forum}. Wiley Online Library, 2023.

\bibitem[Metzer et~al.(2023)Metzer, Richardson, Patashnik, Giryes, and Cohen-Or]{metzer2023latent}
Gal Metzer, Elad Richardson, Or Patashnik, Raja Giryes, and Daniel Cohen-Or.
\newblock Latent-nerf for shape-guided generation of 3d shapes and textures.
\newblock In \emph{CVPR}, pages 12663--12673, 2023.

\bibitem[Mitchel(2022)]{mitchel2022ext}
Thomas~W. Mitchel.
\newblock \emph{Extending Convolution Through Spatially Adaptive Alignment}.
\newblock PhD thesis, Johns Hopkins University, 2022.

\bibitem[Mitchel et~al.(2020)Mitchel, Brown, Koller, Weyrich, Rusinkiewicz, and Kazhdan]{mitchelExtConv2020}
Thomas~W. Mitchel, Benedict Brown, David Koller, Tim Weyrich, Szymon Rusinkiewicz, and Michael Kazhdan.
\newblock Efficient spatially adaptive convolution and correlation, 2020.

\bibitem[Mitchel et~al.(2021{\natexlab{a}})Mitchel, Kim, and Kazhdan]{mitchel2021field}
Thomas~W Mitchel, Vladimir~G Kim, and Michael Kazhdan.
\newblock Field convolutions for surface cnns.
\newblock In \emph{ICCV}, pages 10001--10011, 2021{\natexlab{a}}.

\bibitem[Mitchel et~al.(2021{\natexlab{b}})Mitchel, Rusinkiewicz, Chirikjian, and Kazhdan]{mitchel2021echo}
Thomas~W Mitchel, Szymon Rusinkiewicz, Gregory~S Chirikjian, and Michael Kazhdan.
\newblock Echo: Extended convolution histogram of orientations for local surface description.
\newblock In \emph{Computer Graphics Forum}, pages 180--194. Wiley Online Library, 2021{\natexlab{b}}.

\bibitem[Mitchel et~al.(2022)Mitchel, Aigerman, Kim, and Kazhdan]{mitchel2022mobius}
Thomas~W Mitchel, Noam Aigerman, Vladimir~G Kim, and Michael Kazhdan.
\newblock M{\"o}bius convolutions for spherical cnns.
\newblock In \emph{ACM SIGGRAPH 2022 Conference Proceedings}, pages 1--9, 2022.

\bibitem[Nikankin et~al.(2022)Nikankin, Haim, and Irani]{nikankin2022sinfusion}
Yaniv Nikankin, Niv Haim, and Michal Irani.
\newblock Sinfusion: Training diffusion models on a single image or video.
\newblock \emph{arXiv preprint arXiv:2211.11743}, 2022.

\bibitem[Oechsle et~al.(2019)Oechsle, Mescheder, Niemeyer, Strauss, and Geiger]{oechsle2019texture}
Michael Oechsle, Lars Mescheder, Michael Niemeyer, Thilo Strauss, and Andreas Geiger.
\newblock Texture fields: Learning texture representations in function space.
\newblock In \emph{ICCV}, pages 4531--4540, 2019.

\bibitem[Pavllo et~al.(2021)Pavllo, Kohler, Hofmann, and Lucchi]{pavllo2021learning}
Dario Pavllo, Jonas Kohler, Thomas Hofmann, and Aurelien Lucchi.
\newblock Learning generative models of textured 3d meshes from real-world images.
\newblock In \emph{ICCV}, pages 13879--13889, 2021.

\bibitem[Poole et~al.(2022)Poole, Jain, Barron, and Mildenhall]{poole2022dreamfusion}
Ben Poole, Ajay Jain, Jonathan~T Barron, and Ben Mildenhall.
\newblock Dreamfusion: Text-to-3d using 2d diffusion.
\newblock \emph{arXiv preprint arXiv:2209.14988}, 2022.

\bibitem[Qian et~al.(2023)Qian, Mai, Hamdi, Ren, Siarohin, Li, Lee, Skorokhodov, Wonka, Tulyakov, et~al.]{qian2023magic123}
Guocheng Qian, Jinjie Mai, Abdullah Hamdi, Jian Ren, Aliaksandr Siarohin, Bing Li, Hsin-Ying Lee, Ivan Skorokhodov, Peter Wonka, Sergey Tulyakov, et~al.
\newblock Magic123: One image to high-quality 3d object generation using both 2d and 3d diffusion priors.
\newblock \emph{arXiv preprint arXiv:2306.17843}, 2023.

\bibitem[Richardson et~al.(2023)Richardson, Metzer, Alaluf, Giryes, and Cohen-Or]{richardson2023texture}
Elad Richardson, Gal Metzer, Yuval Alaluf, Raja Giryes, and Daniel Cohen-Or.
\newblock Texture: Text-guided texturing of 3d shapes.
\newblock \emph{arXiv preprint arXiv:2302.01721}, 2023.

\bibitem[Rombach et~al.(2022)Rombach, Blattmann, Lorenz, Esser, and Ommer]{rombach2022high}
Robin Rombach, Andreas Blattmann, Dominik Lorenz, Patrick Esser, and Bj{\"o}rn Ommer.
\newblock High-resolution image synthesis with latent diffusion models.
\newblock In \emph{CVPR}, pages 10684--10695, 2022.

\bibitem[Ruben~Wiersma(2020)]{Wiersma2020}
Klaus~Hildebrandt Ruben~Wiersma, Elmar~Eisemann.
\newblock Cnns on surfaces using rotation-equivariant features.
\newblock \emph{ACM TOG}, 39\penalty0 (4), 2020.

\bibitem[Shaham et~al.(2019)Shaham, Dekel, and Michaeli]{shaham2019singan}
Tamar~Rott Shaham, Tali Dekel, and Tomer Michaeli.
\newblock Singan: Learning a generative model from a single natural image.
\newblock In \emph{ICCV}, pages 4570--4580, 2019.

\bibitem[Sharp et~al.(2022)Sharp, Attaiki, Crane, and Ovsjanikov]{sharp22diffusionnet}
Nicholas Sharp, Souhaib Attaiki, Keenan Crane, and Maks Ovsjanikov.
\newblock Diffusionnet: Discretization agnostic learning on surfaces.
\newblock \emph{ACM TOG}, XX\penalty0 (X), 2022.

\bibitem[Siddiqui et~al.(2022)Siddiqui, Thies, Ma, Shan, Nie{\ss}ner, and Dai]{siddiqui2022texturify}
Yawar Siddiqui, Justus Thies, Fangchang Ma, Qi Shan, Matthias Nie{\ss}ner, and Angela Dai.
\newblock Texturify: Generating textures on 3d shape surfaces.
\newblock In \emph{ECCV}, pages 72--88. Springer, 2022.

\bibitem[Singh and Krishnan(2020)]{singh2020filter}
Saurabh Singh and Shankar Krishnan.
\newblock Filter response normalization layer: Eliminating batch dependence in the training of deep neural networks.
\newblock In \emph{CVPR}, pages 11237--11246, 2020.

\bibitem[Song et~al.(2020)Song, Meng, and Ermon]{song2020denoising}
Jiaming Song, Chenlin Meng, and Stefano Ermon.
\newblock Denoising diffusion implicit models.
\newblock \emph{arXiv preprint arXiv:2010.02502}, 2020.

\bibitem[Svitov et~al.(2023)Svitov, Gudkov, Bashirov, and Lempitsky]{svitov2023dinar}
David Svitov, Dmitrii Gudkov, Renat Bashirov, and Victor Lempitsky.
\newblock Dinar: Diffusion inpainting of neural textures for one-shot human avatars.
\newblock In \emph{ICCV}, pages 7062--7072, 2023.

\bibitem[Tang and He(2023)]{tang2023textguided}
Zhibin Tang and Tiantong He.
\newblock Text-guided high-definition consistency texture model, 2023.

\bibitem[Wang et~al.(2023)Wang, Kanakis, Schindler, Gool, and Obukhov]{wang2023breathing}
Tianfu Wang, Menelaos Kanakis, Konrad Schindler, Luc~Van Gool, and Anton Obukhov.
\newblock Breathing new life into 3d assets with generative repainting, 2023.

\bibitem[Wang et~al.(2022)Wang, Bao, Zhou, Chen, Chen, Yuan, and Li]{wang2022sindiffusion}
Weilun Wang, Jianmin Bao, Wengang Zhou, Dongdong Chen, Dong Chen, Lu Yuan, and Houqiang Li.
\newblock Sindiffusion: Learning a diffusion model from a single natural image.
\newblock \emph{arXiv preprint arXiv:2211.12445}, 2022.

\bibitem[Wang et~al.(2004)Wang, Bovik, Sheikh, and Simoncelli]{wang2004image}
Zhou Wang, Alan~C Bovik, Hamid~R Sheikh, and Eero~P Simoncelli.
\newblock Image quality assessment: from error visibility to structural similarity.
\newblock \emph{IEEE transactions on image processing}, 13\penalty0 (4):\penalty0 600--612, 2004.

\bibitem[Wei et~al.(2023)Wei, Wang, Feng, Lin, and Yap]{wei2023taps3d}
Jiacheng Wei, Hao Wang, Jiashi Feng, Guosheng Lin, and Kim-Hui Yap.
\newblock Taps3d: Text-guided 3d textured shape generation from pseudo supervision.
\newblock In \emph{CVPR}, pages 16805--16815, 2023.

\bibitem[Wiersma et~al.(2022)Wiersma, Nasikun, Eisemann, and Hildebrandt]{wiersma2022deltaconv}
Ruben Wiersma, Ahmad Nasikun, Elmar Eisemann, and Klaus Hildebrandt.
\newblock Deltaconv: anisotropic operators for geometric deep learning on point clouds.
\newblock \emph{ACM Transactions on Graphics (TOG)}, 41\penalty0 (4):\penalty0 1--10, 2022.

\bibitem[Wu et~al.(2023)Wu, Liu, Vondrick, and Zheng]{wu2023sin3dm}
Rundi Wu, Ruoshi Liu, Carl Vondrick, and Changxi Zheng.
\newblock Sin3dm: Learning a diffusion model from a single 3d textured shape.
\newblock \emph{arXiv preprint arXiv:2305.15399}, 2023.

\bibitem[Wu et~al.(2019)Wu, Qi, and Fuxin]{wu2019pointconv}
Wenxuan Wu, Zhongang Qi, and Li Fuxin.
\newblock Pointconv: Deep convolutional networks on 3d point clouds.
\newblock In \emph{CVPR}, pages 9621--9630, 2019.

\bibitem[Yu et~al.(2021)Yu, Dong, Peers, and Tong]{yu2021learning}
Rui Yu, Yue Dong, Pieter Peers, and Xin Tong.
\newblock Learning texture generators for 3d shape collections from internet photo sets.
\newblock In \emph{BMVC}, 2021.

\bibitem[Yu et~al.(2023)Yu, Dai, Li, Ma, Liu, and Qi]{yu2023texture}
Xin Yu, Peng Dai, Wenbo Li, Lan Ma, Zhengzhe Liu, and Xiaojuan Qi.
\newblock Texture generation on 3d meshes with point-uv diffusion.
\newblock In \emph{ICCV}, pages 4206--4216, 2023.

\bibitem[Zhang et~al.(2018)Zhang, Isola, Efros, Shechtman, and Wang]{zhang2018unreasonable}
Richard Zhang, Phillip Isola, Alexei~A Efros, Eli Shechtman, and Oliver Wang.
\newblock The unreasonable effectiveness of deep features as a perceptual metric.
\newblock In \emph{CVPR}, pages 586--595, 2018.

\bibitem[Zhang et~al.(2023)Zhang, Zhang, Chacko, Xu, Song, Yang, and Feng]{zhang2023getavatar}
Xuanmeng Zhang, Jianfeng Zhang, Rohan Chacko, Hongyi Xu, Guoxian Song, Yi Yang, and Jiashi Feng.
\newblock Getavatar: Generative textured meshes for animatable human avatars.
\newblock In \emph{ICCV}, pages 2273--2282, 2023.

\bibitem[Zhuang et~al.(2022)Zhuang, Abnar, Gu, Schwing, Susskind, and Bautista]{zhuang2022diffusion}
Peiye Zhuang, Samira Abnar, Jiatao Gu, Alex Schwing, Joshua~M Susskind, and Miguel~{\'A}ngel Bautista.
\newblock Diffusion probabilistic fields.
\newblock In \emph{ICLR}, 2022.

\end{thebibliography}
